%% file: ms.tex
\documentclass[10pt,journal,compsoc]{IEEEtran}
\pdfoutput=1
\ifCLASSOPTIONcompsoc
  \usepackage[nocompress]{cite}
\else
  \usepackage{cite}
\fi

\usepackage[british]{babel}
\usepackage[ruled]{algorithm2e} 

\usepackage{stfloats}
\usepackage{subcaption}

\usepackage{array}
\usepackage{calc}
\usepackage{amsmath}
\allowdisplaybreaks
\usepackage{amsthm}
\usepackage{booktabs}
\usepackage{mathtools}
\usepackage{stackengine}
\usepackage{multicol}
\usepackage{pifont}

\SetKwInOut{Parameter}{parameters}
\SetKw{Continue}{continue}
\SetKw{Break}{break}

\newcolumntype{C}[1]{>{\hangindent=1em\centering\arraybackslash}m{#1}}
\newcolumntype{R}[1]{>{\hangindent=1em\raggedleft\arraybackslash}m{#1}}
\newcolumntype{L}[1]{>{\hangindent=1em\raggedright\arraybackslash}m{#1}}

\newtheorem{theorem}{Property}
\newtheorem{definition}{Definition}
\newtheorem{corollary}{Corollary}[theorem]
\newtheorem*{remark}{Remark}

\ifCLASSINFOpdf
   \usepackage[pdftex]{graphicx}
\else
\fi
\usepackage{float}

\begin{document}
%
\title{A Bayesian Approach to Rule Mining}
%
%
%
%

\author{Luis~Ignacio~Lopera~Gonz\'alez,~\IEEEmembership{Member,~IEEE,}
		Adrian~Derungs,~\IEEEmembership{Member,~IEEE,}
        Oliver~Amft,~\IEEEmembership{Member,~IEEE,}
\IEEEcompsocitemizethanks{\IEEEcompsocthanksitem 
	L.I. Lopera Gonz\'alez, A. Derungs, and O. Amft are with the 
	Chair of Digital Helath at the Erlangen-N\"urnberg Univesity,
    Erlangen, Germany, 91052.\protect\\
E-mail: luis.i.lopera@fau.de}
\thanks{
}}

%
%

\markboth{}%
{Lopera \MakeLowercase{\textit{et al.}}: \title}
%



\IEEEtitleabstractindextext{%
\begin{abstract}
In this paper, we introduce the increasing belief criterion in association rule
mining. The criterion uses a recursive application of Bayes' theorem to compute
a rule's belief. Extracted rules are required to have their belief increase with
their last observation. We extend the taxonomy of association rule mining
algorithms with a new branch for Bayesian rule mining~(BRM), which uses increasing
belief as the rule selection criterion. In contrast, the well-established
frequent association rule mining~(FRM) branch relies on the minimum-support
concept to extract rules.

We derive properties of the increasing belief criterion, such as the increasing
belief boundary, no-prior-worries, and conjunctive premises. Subsequently, we
implement a BRM algorithm using the increasing belief criterion, and illustrate
its functionality in three experiments: (1)~a proof-of-concept to illustrate BRM
properties, (2)~an analysis relating socioeconomic information and chemical exposure
data, and (3)~mining behaviour routines in patients undergoing neurological
rehabilitation. We illustrate how BRM is capable of extracting rare rules and
does not suffer from support dilution. Furthermore, we show that BRM focuses on
the individual event generating processes, while FRM focuses on their
commonalities. We consider BRM's increasing belief as an alternative criterion
to thresholds on rule support, as often applied in FRM, to determine rule
usefulness.
\end{abstract}

}

\maketitle

\IEEEdisplaynontitleabstractindextext

%

\input{introduction}

\input{related_work}

\input{properties}
\input{methodology}

\input{proof_of_concept}

\input{data_base_mining}
\input{time_series_mining}

\input{conlusions}
\section*{Acknowledgments}
We are thankful to the study participants and the clinicians at the Reha Rheinfelden. The authors thank Dr. Julia Seiter for supporting the study data recordings.

\bibliographystyle{IEEEtran}
\bibliography{bibliography}
\vfill\break

\begin{IEEEbiography}[{\includegraphics[width=1in,height=1.25in,clip,keepaspectratio]{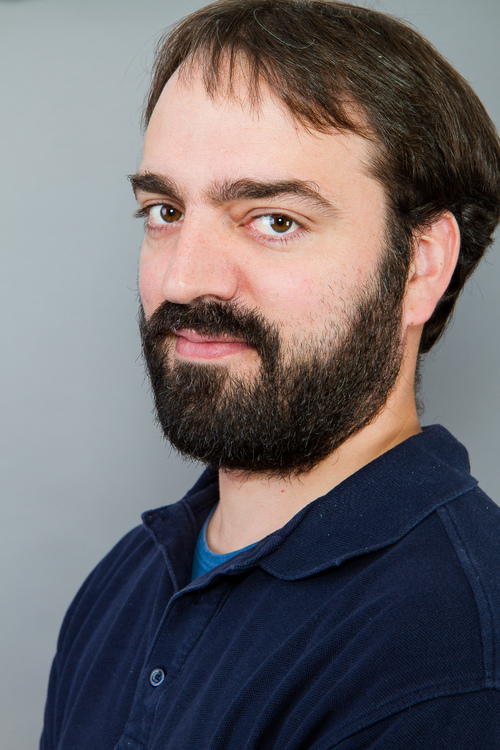}}]{{Luis Ignacio} {Lopera Gonz\'{a}lez} }
received the MSc. degree in Electronic engineering and computer science from Los
Andes University, Bogot\'a, Colombia, in 2005 and the Dr. rer. nat. degree from
the University of Passau, Passau, Germany in 2018.  He is currently working at
the Chair of Digital Health at the Friedrich-Alexander University
Erlangen-N\"urenberg, Erlangen, Germany. His current research interest include
rule mining algorithm development, and self commissioning systems in medical
applications.
\end{IEEEbiography}
%
\begin{IEEEbiography}[{\includegraphics[width=1in,height=1.25in,clip,keepaspectratio]{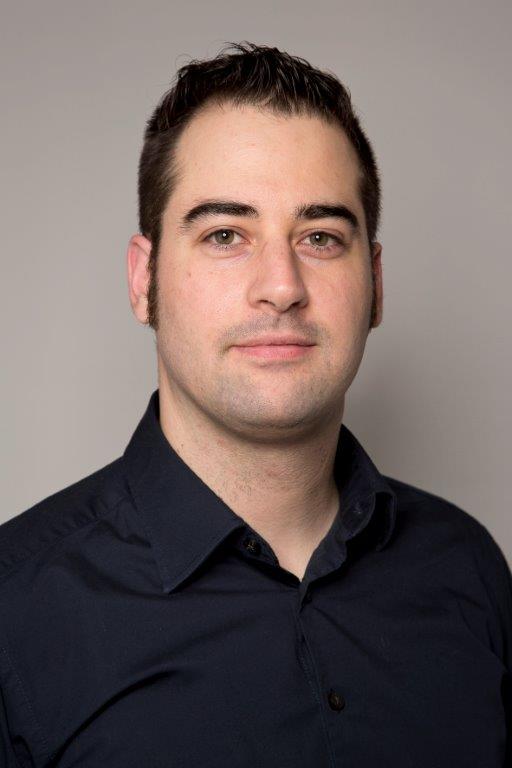}}]{Adrian Derungs}
received the MSc. degree in Biomedical Engineering from the University of Bern,
Bern, Switzerland in 2012. He is currently working toward the PhD degree with
the Chair of Digital Health at the Friedrich-Alexander University
Erlangen-N{\"u}rnberg, Erlangen, Germany. His current research interest include,
wearable sensor technology and computing, machine learning, biomechanics and
biomarker personalisation, and the digitalisation of healthcare.
\end{IEEEbiography}
%
%
\begin{IEEEbiography}[{\includegraphics[width=1in,height=1.25in,clip,keepaspectratio]{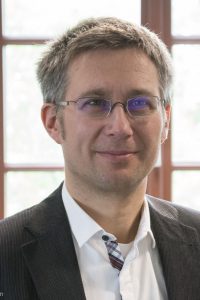}}]{Oliver Amft}
  received the Dipl.-Ing. (M.Sc.) from Chemnitz Technical University in 1999 and
the Dr. sc. ETH (Ph.D.) from ETH Zurich in 2008, both in Electrical
Engineering and Information Technology. In 2011, he obtained the university
teaching qualification from TU Eindhoven. Until 2004, he was a R\&D project
manager with ABB, Inc. Oliver Amft was an assistant professor at TU Eindhoven
between 2009 and 2013, tenured since 2011. In 2014, he was appointed full
professor and established the Chair of Sensor Technology at University of
Passau, Faculty of Computer Science and Mathematics. Since 2017, Oliver Amft
is a full professor and director of the Chair of Digital Health at
Friedrich-Alexander University Erlangen-N\"urnberg.
\end{IEEEbiography}
\vfill

\newpage
\onecolumn

\include{apendix}

\end{document}

%% file: introduction.tex
\IEEEraisesectionheading{\section{Introduction}} 
\IEEEPARstart{A}{ssociation} rules can model a process by describing the
relationship between its variables. In a dynamic process, for example, a rule
states that a change on an input will cause a change on an output. As the
process' evolution is stored in a dataset, rule mining can extract the original
relationship between the process' inputs and outputs. The frequently applied
paradigm for mining rules is frequent association rule mining~(FRM). Rules
extracted by FRM, e.g., $a \rightarrow b$, have their support greater than a set
threshold called minimum-support. The literature defines support as in
Equation~\ref{eq_support}, where $\#(a \rightarrow b )$ is the rule's occurrence
count and $|D|$ is the dataset
size~\cite{Agrawal1993MiningAssociationRulesSets}. 
\begin{equation}
	\mathrm{Support}(a \rightarrow b )=\frac{\#(a \rightarrow b )}{|D|}
 	\label{eq_support}
\end{equation}
 
The following thought experiment illustrates one of FRM's limitations. Suppose
we use all supermarket receipts from the winter holidays for rule mining. Then,
we would see rules that associate the ingredients used for winter holiday meals.
However, if we now consider a year worth of receipts from the same supermarket,
then, the winter holiday meal ingredients would not have enough support to be
extracted. In other words, the minimum-support threshold used to extract rules
in a small dataset will not work in  an extended  version of the dataset due to
the definition of rule support. We refer to FRM's dependency on the dataset size
as support dilution.
 
Another FRM limitation appears when a dataset contains multiple processes. FRM
has the implicit assumption that all processes generate symbols at the same
rate. However, in practice, processes can generate symbols at different rates.
For example, people in Germany occasionally buy white sausages, but when they
do, they always buy wheat beer. So the rule `if white sausage then wheat beer' is
a rare rule when compared to frequent rules, e.g.,  `if milk then eggs'. FRM can
extract rare rules by using a low minimum-support threshold. Unfortunately, 
spurious symbol associations may create unwanted rules that FRM's threshold
cannot eliminate. Filtering out unwanted rules for FRM has been
addressed in the past
\cite{LoperaGonzalez2016MiningHierarchicalRelationsBuilding,
	Liu2018RareItemsetsMiningAlgorithm}. However, we observed that the processing
required to separate rules does not generalise.
 
To bypass FRM's support dilution and rare rule extraction limitations, we
propose to exploit the belief concept of Bayesian filtering and derive an
increasing belief criterion. We drew inspiration from Price’s account of Bayes
views on updating beliefs~\cite{Bayes1763EssaySolvingProblemDoctrine}. Price
illustrates the idea with an example of a group of cavemen coming out into the
world for the first time. The first thing the cavemen will probably notice is
the sun, and how the sun moves through the sky until it disappears. At this
point, the cavemen are unaware if the sun will appear again. As the sun rose and
set during the next days, the cavemen will update their belief about the sun
setting and rising as a defining feature of the way the world works outside the
cave.
 
In this work, we present the following contributions: 
\\(1)~We extend the taxonomy
of association rule mining by a branch for Bayesian rule mining~(BRM), which
uses increasing belief as rule mining criterion. 
\\
(2)~We introduce the increasing belief criterion, derive key properties, and show their
application. 
\\
(3)~We implement an exhaustive search BRM
algorithm and demonstrate its viability as solution for support dilution and 
rare rule extraction, by extracting rules in three datasets: a synthetic time
series, a dataset linking socio-economic variables with chemical-exposure
information, and a dataset of daily behaviour routine annotations of patients
with hemiparesis. In each experiment, we compared rules extracted by BRM and FRM
and evaluated their quality and usefulness for each application.

%% file: related_work.tex
\begin{figure*}[th!]
	\includegraphics[width=\linewidth]{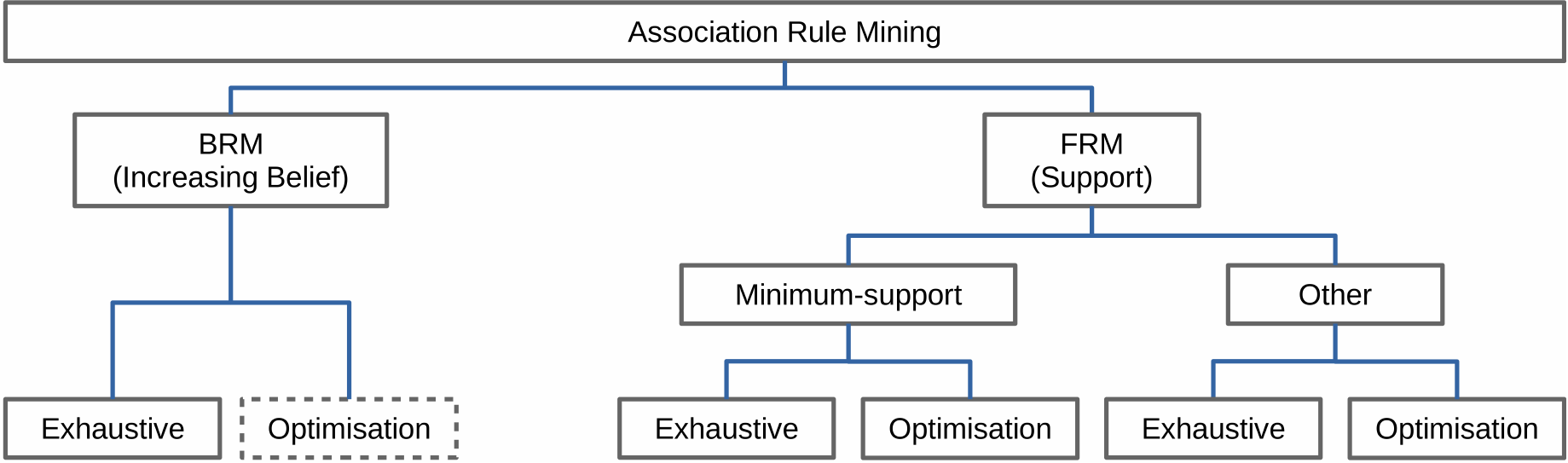}
	\caption[Extended taxonomy of association rule mining]{Extended taxonomy of association rule
	mining. We introduce a new branch for Bayesian rule mining~(BRM), which employs
	increasing belief as rule selection criteria. The method presented in this
	paper is an exhaustive search algorithm. Nevertheless, further algorithms based on the increasing belief criterion are conceivable. The taxonomy expands on Beiranvand et
	al.~\cite{Beiranvand2014MultiobjectivePSOAlgorithmMining}. }
	\label{fig_taxonomy}
\end{figure*}  

\section{Related Work} Association rule mining algorithms can be organised into
branches by their main rule selection criteria.  We consider two algorithm
branches: the FRM branch, which uses support
to extract rules, and the BRM branch, which uses
increasing belief. Figure~\ref{fig_taxonomy} illustrates the first four levels
of our proposed taxonomy. The use of minimum-support determines the first subclasses level for FRM, and the method used for finding frequent symbol sets
defines the next level. Padillo et
al.~\cite{Padillo2017MiningAssociationRulesBig} provided a comprehensive list of
the FRM algorithms, and we expanded on the FRM taxonomy of optimisation
algorithms that use minimum-support proposed by Beiranvand et
al.~\cite{Beiranvand2014MultiobjectivePSOAlgorithmMining}. To our knowledge, and
except for our initial work on
BRM~\cite{LoperaGonzalez2018MiningFunctionalStructuralRelationships}, the
entirety of research on association rule mining belongs to the FRM branch. In the BRM branch, the
method used to find relevant symbol sets determines the first subclass level. We
have added the optimisation block as a place holder for future research.

Below, we elaborate on the FRM methods along the taxonomy. The related work is
organized in three subsections. First, we follow the taxonomy illustrated in
Fig.~\ref{fig_taxonomy} to present FRM methods. Next, we provide examples of
work that considered the challenge of rare rule extraction. Finally, we describe
our previous work in BRM.

\subsection{Frequent Association Rule Mining Branch}
\subsubsection{Minimum-support, Exhaustive}
An exhaustive search algorithm has to traverse the lattice created from the
symbol's power set. The downward-closure
property~\cite{Tan2005IntroductionDataMining} can be used to facilitate lattice
traversal. The property states that the support of any symbol set is less or
equal to the support of any of its possible subsets. Therefore, exhaustive
search algorithms can ignore entire branches of the symbol set lattice when a
given set does not pass the minimum-support threshold.

The use of the downward-closure property divides FRM's exhaustive search class
into two subclasses. \emph{OPUS} is an example of exhaustive search algorithms
that do not use the downward-closure property to traverse the symbol lattice.
Webb~\cite{Webb1995OPUSEfficientAdmissibleAlgorithm} created \emph{OPUS} as a
frequent-set mining algorithm. Subsequently,
Webb~\cite{Webb2000EfficientSearchAssociationRulesa} presented an extension that
converted \emph{OPUS} into a rule mining algorithm. The
\emph{Apriori}~\cite{Agrawal:1994:FAM:645920.672836} and
\emph{ECLAT}~\cite{Zaki2000ScalableAlgorithmsAssociationMininga} algorithms use
the downward-closure property and define two distinct approaches on how to
traverse the symbol set lattice.  The \emph{Apriori} algorithm uses a
breadth-first approach. The algorithm starts with all one symbol sets and
traverses the lattice upwards using the downward-closure property.  The
\emph{ECLAT} algorithm uses a depth-first approach, starting with one symbol and
traversing the lattice in depth using the downward-closure property to avoid
infrequent branches. Exhaustive search algorithms in BRM can use the
premise-conjunction property~(Property~\ref{prop_premise_conjunction_rules}) to
simplify the search for premise symbol sets with more than one element. Our
proposed algorithm uses a breadth-first approach to traverse the symbol lattice.

\subsubsection{Minimum-support, Optimisation} Evolutionary algorithms use
genetic optimization or swarms to find the most common symbol sets. The general
approach is to write maximization objectives which are usually functions of
support and confidence~\cite{Bashir2009FastAlgorithmsMiningInteresting,
	Djenouri2018NewFrameworkMetaheuristicbasedFrequent,
	Tahyudin2017CombinationEvolutionaryAlgorithmMethod}. However, most approaches
still need a minimum-support threshold to select rules. BRM opens a new area of
research, as the increasing belief criterion can be integrated with evolutionary
algorithms, thus replacing support and related measures of rule interest.

Rule grammars and machine learning approaches have been presented as
alternatives of the downward-closure property to restrict rule search space. For
example, Padillo et al.~\cite{Padillo2017MiningAssociationRulesBig} uses rule
grammar and map reduce to optimize the mining process in large datasets. Rule
grammars can be used seamlessly with increasing belief, and rule belief can be
maximized using map reduce by means of Prop.~\ref{prop_inc_belief}. In
timeseries, Guillam-Bert et
al.~\cite{GuillameBert2012LearningTemporalAssociationRules} proposed the TITARl
algorithm. They used decision trees to improve rule description and specificity,
specially when regarding time.

\subsubsection{Other FRM Algorithms} Other algorithms in FRM do not use
minimum-support as part of the rule selection process. ARMGA and EARMGA are
optimisation algorithms proposed by Yan et
al.~\cite{Yan2009GeneticAlgorithmbasedStrategyIdentifying}. (E)ARMGA uses a
genetic programming approach with relative confidence as the fitness function.
The algorithm searches for the best $k$ rules until it reaches a maximum number
of generations or the difference between the relative confidence of the best and
worst individuals is less than a parameter $\alpha$. By avoiding the minimum-support
threshold, (E)ARMGA does not suffer from support dilution. However, as data is
added that does not maintain the distribution of initial rule's symbols,
relative confidence increases and tends to one, regardless of the relative
frequencies of the premise and conclusion sets. Furthermore, the parameter
$\alpha$ is difficult to set without solving the rule mining problem, i.e.,
without doing a parametric search.

Bashir et al.~\cite{Bashir2009FastAlgorithmsMiningInteresting} proposed an
exhaustive search algorithm, which starts by selecting $n$ symbol sets. Then,
their algorithm selects the smallest support values of the chosen symbols sets,
and prunes the search space using the downward-closure property. The search
process is repeated looking for sets with higher support. As there are always
$n$ selected sets, the value for pruning the search spaced is set to the
smallest support value.

\subsection{Rare Rule Extraction} When researchers applied rule mining to
practical problems, they noted that the most-frequent rules were not always the
most interesting. For example, in activities of daily living, some patterns,
like walking and talking,  occur frequently, but do not yield interesting rules.
Therefore, several rule interest metrics were proposed, including
lift~\cite{Brin1997MarketBasketsGeneralizingAssociation},
conviction~\cite{Brin1997DynamicItemsetCountingImplication}, among others, to
prune the final rule set. Nevertheless, rule interest metrics are either based
on support or depend on frameworks that use minimum-support to select the
initial candidate rule set. We argue that the use of support to preselect rules
causes the mining algorithm to miss rare but essential rules. For example,
consider a ceiling lamp controlled by motion and the absence of daylight. While
motion may often trigger the lamp, in comparison, absence of daylight will only
trigger the lamp a couple of times a day.  As a result, finding the relation
between the absence of daylight and the lamp being triggered is difficult based
on support alone~\cite{LoperaGonzalez2016MiningHierarchicalRelationsBuilding}.

The common tactic to solve rare rule problems is to use a sufficiently low
minimum-support threshold. Then, algorithms use rule interest metrics to extract
interesting and rare rules. For example, Liu et
al.~\cite{Liu2016MiningIntricateTemporalRules} proposed a methodology to extract
intricate activity patterns from timeseries. Their method searches for rules in
a region bounded by minimum-support and confidence thresholds. The final rules
are selected using a threshold on information gain. A similar approach was
implemented by Liu~et al.~\cite{Liu2018RareItemsetsMiningAlgorithm}. They
proposed a support band to determine a rule's rarity. In contrast,
Srinivasan~{et al.}\cite{Srinivasan2018RuleSelectorSelectingConditionalAction}
proposed a method for extracting conditional action rules using four new rule
selection criteria. However, their framework uses a minimum-support threshold to
mine frequent symbol sets.

\subsection{Bayesian Rule Mining Branch} We expand on our initial BRM
work~\cite{LoperaGonzalez2018MiningFunctionalStructuralRelationships} by
refining the theoretical derivation and mined rule handling. Moreover, we
introduced two additional properties of increasing belief, add an optional
control parameter, extend the algorithm for database mining, and present three
new evaluation scenarios.

%% file: properties.tex
\section{Increasing Belief Criterion}
We defined belief in a rule as the recursive application of the Bayes theorem,
as shown in Definition~\ref{def_belief}. The initial belief $B_1(r)$ of a rule
$r = a \rightarrow b$ is the probability of observing the conclusion $b$ given
that the premise $a$ was observed, i.e., $P(b|a)$. Using the Bayes theorem,
$B_1(r)$ is calculated using the probability $P(a|b)_1$ and the prior $p$. In
general, $P(a|b)_k$ can be estimated by the ratio $k / \#b_k$, where $k$ is the
$k$th observation of the rule $r$ and $\#b_k$ is the number of times $b$ is
observed at the $k$th observation of $r$. Thus, in the initial state, $P(a|b)_1
= 1/\#b_1$. The prior $p$ is an algorithm parameter. $P(b|a)_k'$ and $p'$ denote
the compliments of $P(b|a)_k$ and $p$, respectively. Finally, in the recursive
evaluation of $B_k(r)$, the previous belief evaluation~($B_{k-1}(r)$), replaces
the prior $p$. A common criticism of Bayesian methods is the challenge of
selecting the correct prior probabilities. However, as explained in 
Prop.~\ref{prop_no_prior},  a prior in range $(0, 1)$ has no effect on the
rule selection process.
\begin{definition}[Belief] 
	\label{def_belief}
	Recursive definition of belief.
	\begin{equation*}
	\begin{aligned}
	B(r)_1 &= \frac{P(a|b)_1\cdot p}{P(a|b)_1\cdot p + P(b|a)_1'\cdot p'} \\
	B(r)_k &= \frac{P(a|b)_k\cdot B(r)_{k-1}}{P(a|b)_k\cdot B(r)_{k-1} + P(a|b)_k'\cdot B(r)_{k-1}'}
	\end{aligned}
	\end{equation*}
\end{definition}

The increasing belief criterion requires that a rule's belief does not decrease
with respect to the previous observation, as shown in
Def.~\ref{def_inc_belief}.
\begin{definition}[Increasing Belief Criterion]
	\label{def_inc_belief} 
	A rule has increasing belief, at observation $k$, when $B(r)_k \geq
	B(r)_{k-1}$.
\end{definition}

\subsection{Properties and Proofs} Here, we present the properties we derived
from the rule selection criterion of increasing belief~(Def.~\ref{def_inc_belief}).

\begin{theorem}[Increasing Belief Boundary]
	\label{prop_inc_belief}
	The belief of a rule $r = a \rightarrow b$ will increase or stay constant if
	the conditional probability $P(a | b)$ is equal or greater than 0.5.
\end{theorem}

\begin{proof}
	The increasing belief boundary property is proven by simplifying the expression
	$B(r)_k \geq B(r)_{k-1} $. 
	\centering 
	\small
	\begin{align*}
		B(r)_{k-1} &\leq B(r)_k &&  \text{Def.~\ref{def_inc_belief}}\\
		B(r)_{k-1} &\leq \frac{P(a | b) B(r)_{k-1}}{P(a | b) B(r)_{k-1} 
			+     
			P(a | b)' B(r)'_{k-1}} && \text{Def.~\ref{def_belief}}\\
		1 &\leq \frac{P(a | b)}{P(a | b) B(r)_{k-1} +     
			P(a | b)' B(r)'_{k-1}}\\
		P(a | b) &\geq P(a | b) B(r)_{k-1} + P(a | b)' B(r)'_{k-1}\\ 
		P(a | b)\left(1 - B(r)_{k-1}\right) &\geq P(a | b)' B(r)'_{k-1}\\ 
		P(a | b)B(r)'_{k-1} &\geq P(a | b)' B(r)'_{k-1} \\
		P(a | b) &\geq P(a | b)' = 1 - P(a | b)  \\
		2P(a | b) &\geq 1 \\
		P(a | b) &\geq 0.5
	\end{align*}
\end{proof}

\begin{theorem}[Saturation]
	\label{prop_saturation} If a rule reaches belief of one, new independent symbol
	observations will not alter the rule's belief.
\end{theorem}

\begin{proof}
	Assume the belief $B(r)_{k-1}$ is one.
	\begin{equation*}
		\begin{aligned}
			B(r)_k &= \frac{P(a | b)B(r)_{k-1}}{P(a | b) B(r)_{k-1} +     
			P(a | b)' B(r)'_{k-1}} && \text{Def.~\ref{def_belief}} \\
			& \text{Using the assumption} \\
			B(r)_{k-1} &= 1 \wedge B(r)'_{k-1} = 0 \\
			&\Rightarrow\\
			B(r)_k &= \frac{P(a | b)\cdot 1 }{P(a | b) \cdot 1 +     
				P(a | b)'\cdot 0} \\
			B(r)_k &= \frac{P(a | b)}{P(a | b)}\\
			B(r)_k &= 1	
		\end{aligned}
	\end{equation*}

	In saturation, changes to the probability $P(a | b)$, i.e., the ratio between
	rule observations and the conclusion symbol observations, have no effect on the
	rule's belief.
\end{proof}

\begin{corollary}
	Rules created with the first observation of a symbol as conclusion have
	saturated belief.
\end{corollary}
\begin{proof}
	Note that when the rule and its conclusion symbol are observed for the first
	time, the probability $P(a|b)=1$
\begin{equation*}
	P(a|b) = \frac{\#r}{\#b} = \frac{1}{1} = 1 
\end{equation*}
	Replacing the value of $P(a|b)$ in Def.~\ref{def_belief} for $k=1$ yields
	$B(b|a)_1 = 1$. Therefore, the rule is in saturation.
\end{proof}

\begin{corollary}
	A similar saturation effect occurs when the rule's belief is equal to zero.
\end{corollary}
\begin{proof}
	When $B_{k-1}(r) = 0$, then, by Def.~\ref{def_belief}, $B_{k}(r) = 0$,
	regardless of the value of $P(a|b)$.
\end{proof}
\begin{theorem}[Premise Conjunction]
	\label{prop_premise_conjunction_rules} 
	
	If a rule $r$ with a conjunction of symbols as premise has increasing belief,
	its atomic rule constituents also have increasing belief.
\end{theorem}
\begin{proof}
	Suppose that rule $r$ is of the form $(a,b,c) \rightarrow d$ and it has
	increasing belief, but the constituent atomic rule  ${r}_{ad}$ of the form $a
	\rightarrow d$ has not. Let $\#r$, $\#r_{ad}$, $\#r_{bd}$, and $\#r_{cd}$ 
	denote the number of times the rule and its atomic constituents were observed
	in the dataset. Additionally, $\#d$ is the number of times $d$ was observed in
	the dataset.
	\begin{equation*}
		\begin{aligned}
			\frac{\#r}{\#d} &\geq 0.5 && \text{Assumption and} \\
								&	~&& \text{Prop.~\ref{prop_inc_belief}} \\
			\#r &\leq \mathrm{min}\left(\#r_{ad}, \#r_{bd}, \#r_{cd} \right) \\
			\#r &\leq \#r_{ad} \\			
			\frac{\#r}{\#d} &\leq \frac{\#r_{ad}}{\#d} \\
			\frac{\#r}{\#d} &\leq \frac{\#r_{ad}}{\#d} < 0.5 && \mathrm{Prop.~\ref{prop_inc_belief}}\\
			\frac{\#r}{\#d} &< 0.5 \\
		\end{aligned}
	\end{equation*} 
	Contradiction. $\frac{\#r}{\#d}$ cannot~be $<0.5$  and $\geq 0.5$ simultaneously.
\end{proof}

\begin{corollary}
	\label{coly_premise_conjunction_rules}
	A similar argument to conjunctive premises cannot be made for rule conclusions.
	The atomic constituents of a conjunctive conclusion rule $r$ are not required
	to have increasing belief for $r$ to pass the increasing belief criterion. The
	constituent atomic rule frequencies have no effect on the rules' belief.
\end{corollary}
\begin{proof}
	Assume there is  a rule $A \rightarrow B$, where $B$ is a combination of
	symbols occurring only once in the dataset. Therefore, the rule $A \rightarrow
	B$ occurs only once too. Furthermore, $P(A|B)=1$ implying that the rule has
	increasing belief due to Prop.~\ref{prop_inc_belief}. As a result, the
	constituent atomic rule's belief has no relevance for the rule with a
	conjunctive conclusion.
\end{proof}

\begin{theorem}[No-Prior-Worries]
	\label{prop_no_prior} For the increasing belief criterion, the prior parameter
	$p$ has no effect for rule selection and any value in the open interval $(0,1)$
	can be used. Using zero or one as $p$ will cause the rule's belief to
	saturate~(Prop.~\ref{prop_saturation}).
\end{theorem}

\begin{proof}
	We look at the first iteration of the belief evaluation criteria.
	\begin{equation*}
		\begin{aligned}
			 & \text{By \hfil{Defs.}~\ref{def_belief} and~\ref{def_inc_belief}} \\
			p &\leq \frac{P(a|b) \cdot p}{P(a|b) \cdot p + (1 - P(a|b))\cdot(1-p)}	\\
			1 &\leq \frac{P(a|b)}{P(a|b) \cdot p + (1 - P(a|b))\cdot(1-p)}\\
			P(a|b) &\geq P(a|b) \cdot p + (1 - P(a|b))\cdot(1-p)\\
			P(a|b) &\geq P(a|b) \cdot p + 1 - P(a|b) -p + P(a|b) \cdot p\\
			P(a|b) &\geq 2 \cdot P(a|b) \cdot p + 1 - P(a|b) -p \\
			2 \cdot P(a|b) -1  &\geq 2 \cdot P(a|b) \cdot p - p\\
			2 \cdot P(a|b) -1  &\geq p \cdot \left( 2 \cdot P(a|b) - 1\right) \\
			1 &\geq p \\
			  &~\text{Avoid saturation -- Prop.~\ref{prop_saturation}}\\
			1 &> p && 
		\end{aligned}
	\end{equation*}
	
\end{proof}

\begin{remark}
	Although the prior has no effect on rule selection, the prior determines a
	rule's final belief value.
\end{remark}

%% file: methodology.tex
\section{Implementation of the BRM algorithm} Algorithm~\ref{alg_bre} describes
BRM's procedure,  where $D$ is a dataset, and $t$ is a respective record or
observation window in $D$. Observation windows are used to analyse timeseries.
The algorithm parameter $ow$ sets the size of the observation window. During
analysis, the observation window moves in steps of one symbol. With $t$ properly
defined, BRM creates candidate rules from each $t$ using
\texttt{select\_candidate\_rules}.  Candidate rules $CR$ are created in two
different ways depending on whether $D$ is a database or a timeseries. For a
timeseries, and to respect time dependency, the candidate rules are the result
of pairing the first symbol in the observation window with all remaining
symbols. In contrast, for a database, the candidate rules are all pairwise
combinations of symbols in the record. Next,
\texttt{pass\_selection\_criteria(CR)} returns a set of candidate rules that
pass the selection criterion shown in Def.~\ref{def_inc_belief}. The rules
returned by \texttt{pass\_selection\_criteria(CR)} are added to the rule set
$A$. Subsequently, any rule in $CR$ that fails the selection criterion is
removed from $A$. BRM only requires one pass of $D$.

\begin{algorithm}
	\caption{Atomic rule mining}	
	\label{alg_bre}
	\Parameter{$p$, $s=1$, [$ow$]} 
	A = $\emptyset$ \\
	\ForAll{t $\in D$}{
		CR = select\_candidate\_rules($t$)\\
		A.update(pass\_selection\_criteria(CR))\\
		A.remove(!pass\_selection\_criteria(CR))\\
	}
	cross\_check\_rules(A)\\
	quick\_update\_belief(A)\\
	\Return A
\end{algorithm}

BRM only evaluates a rule's increasing belief criterion when the rule is
observed. Therefore, once BRM processed the entire dataset,  rules are checked
for loss of belief and removed if necessary. Loss of belief can occur in two
cases: (1) the rule's belief saturated~(Prop.~\ref{prop_saturation}) and
additional observations did not change the rule's belief, or, (2)~at the end of
the dataset, BRM observed extra conclusion symbols without corresponding rule
observations. Both conditions are checked by the \texttt{cross\_check\_rules}
function.  Using Prop.~\ref{prop_inc_belief}, \texttt{cross\_check\_rules}
evaluates if the rules in $A$ still pass the increasing belief criterion.
Without \texttt{cross\_check\_rules}, the rule set $A$ would depend on the order
of the records. At last, the \texttt{quick\_update\_belief} function checks the
saturated rules left in $A$, i.e., saturated rules that pass the belief
criterion, and updates their beliefs.  The function assumes that all
unassociated conclusions occurred before any of the rule observations. Then, the
belief of a saturated rule is computed using BRM's internal counters. Therefore,
\texttt{quick\_update\_belief} does not require additional passes over dataset $D$.

To add some control to BRM's selectivity, 
we incorporated the parameter Selector $s$, with range [0,1], in the probability estimation $P(a|b)$, as shown in Eq.~\ref{eq_selector}. The Selector $s$ determines the percentage of
unassociated conclusion symbol observations to consider when computing a rule's
belief. For Selector $s=1$, BRM's default behaviour is realised, i.e. all conclusion symbol occurrences are used. In contrast, for Selector $s=0$, all candidate rules will be accepted, as  $P(a|b)=1$ and any rule will have maximum belief.
\begin{equation}
P(a|b) = \frac{\#(a\rightarrow b)}{s\cdot(\#b - \#(a\rightarrow b)) + \#(a\rightarrow b)}
\label{eq_selector}
\end{equation}

Alg.~\ref{alg_bre} extracts atomic rules, i.e., rules which have one symbol in their
premise and conclusion respectively. More complex rules are
found using Alg.~\ref{alg_cpe}, which is based on the premise-conjunction
property (Prop.~\ref{prop_premise_conjunction_rules}).
For example, assume the rule
$(a,b,c) \rightarrow d$ has increasing belief. Then, it follows that the rules
$a \rightarrow d$, 
$b \rightarrow d$, 
$c \rightarrow d$, 
$(a,b)\rightarrow d$, 
$(a,c)\rightarrow d$, and 
$(b,c)\rightarrow d$, also have increasing
belief. As a result, Prop.~\ref{prop_premise_conjunction_rules} is used to
narrow the search space of conjunctive premises. However, as shown in
Corollary~\ref{coly_premise_conjunction_rules}, there is no mathematical
requirement for conjunctive conclusions, e.g., $a \rightarrow (b,c,d)$,  to have
atomic constituents with increasing belief. Thus, we argue that it is the
application that will specify how to search for conjunctive conclusions, i.e.,
whether or not to require atomic constituents with increasing belief.

\begin{algorithm}
	\caption{Conjunctive Premise Rule Mining}	
	\label{alg_cpe}
	\Parameter{A}
	$G = \{b: [a_i]~\forall a_i \rightarrow b \in \mathrm{A}\}$\\
	Block = $\emptyset$ \\
	NewRules = $\emptyset$\\
	\ForEach{$b: [a_i] \in  G$}{
		\ForEach{$j \in [2, \dots, |[a_i]|]$}{
			$P = \{\forall \mathbf{a} \in \binom{[a_i]}{j} \wedge \forall \dot{\mathbf{a}}  \in \mathrm{Block} | \dot{\mathbf{a}} \not\subset \mathbf{a} \}$\\
			\If{$P == \emptyset$}{
				\Break}
			\ForEach{$\mathbf{a} \in P$}{
				\uIf{$\mathrm{pass\_selection\_criteria}(\mathbf{a} \rightarrow b)$}{
					NewRules.update($\mathbf{a} \rightarrow b$)	
				}
				\Else{
					Block.update($\mathbf{a}$)}
				
			}
		}	
	}
	\Return NewRules	
\end{algorithm}

Alg.~\ref{alg_cpe} uses a breath-first approach.
First, the function groups all atomics rules in $A$ by their conclusions in the
dictionary $G$, where for a list of rules $[a_i \rightarrow b] \in A$ the
premise $b$ is the dictionary key, and the list of associated premises $[a_i]$
are the values.  Subsequently, for each conclusion $b \in G$ and associated list
of premises $[a_i]$, we explore the different combinations of
$\left(\frac{[a_i]}{j}\right)$ to create new rules with conjunctive premises,
starting with pairwise combinations~($j=2$) and finalizing with the complete
list $[a_i]$. For each value of $j$, a set $P$ is constructed with all possible
$j$-combinations of premises in $[a_i]$, excluding those combinations that
contain elements in set $Block$. Using each $j$-combination $\mathbf{a}$ in $P$ as premise,
e.g., $\mathbf{a} = [a_0, a_1]$ for $j=2$, new rules are constructed with $b$ as the 
conclusion, e.g., $\mathbf{a}
\rightarrow b$. The rules with increasing belief are added to the $NewRules$
set. For rules without increasing belief, the rule's premise $\mathbf{a}$ is
added to the $Block$ set. The search of conjunctive premises for conclusion $b$
ends when the function reaches either the set of all premise symbols or an empty
set $P$ of unblocked combinations. The function finishes after evaluating all
conclusions in $G$ and returns the $NewRules$ set.

%% file: proof_of_concept.tex
\section{Proof of concept}
With the experiment illustrated here, we intend to show how BRM helps to sort
symbols according to their generating process. Specifically, we show that BRM
does not suffer from support dilution. As a result, BRM can extract rare rules
created by a process with rare symbol emission.

\subsection{Methodology} 
We created a timeseries generator to simulate processes emitting common and rare
symbols. Using BRM, we mined the resulting timeseries for atomic rules. With the
mined rules, we constructed graphs, when independent subgraphs formed, each
subgraph was considered a symbol cluster that represented a generating process.

The timeseries generator mixed two processes: (1)~a common process $p_r(t)$ that
frequently emitted random symbols, and (2)~a rare process $p_c(t)$ that
occasionally emitted a specific pattern of symbols denoted as a chain. The
process $p_r(t)$ sampled vocabulary $V_r = {0,1,2,3}$ using a uniform
distribution. In contrast,  $p_c(t)$ used vocabulary $V_c={10,11,12}$  to emit
the chain $10 \rightarrow 11 \rightarrow 12$. The chain symbols were always
emitted in the same order, but the timing between symbols varied uniformly,
sampled from the integer interval [1,10]. The timeseries generator filled the
gaps between $p_c(t)$ emissions with symbols from $p_r(t)$, resulting in a dense
timeseries. Additionally, the timeseries generator used 1000 sampled symbols
from $p_r(t)$ and 20 chains from $p_c(t)$. $p_c(t)$ chains were uniformly
distributed throughout the timeseries and could not overlap.
Eq.~\ref{eq_time_series} shows an excerpt of a generated timeseries $ts$,
with the symbols emitted from $p_c(t)$ highlighted.
\begin{equation}
\label{eq_time_series}
ts = [\cdots, 0, 2, 3,\mathbf{10}, 2, 0, 0,\mathbf{11},
             2, 2, 1, 1,\mathbf{12}, 2, 2, 2,\cdots]
\end{equation}	
 
We processed the timeseries with BRM and created graphs using the mined atomic
rules. To improve the subgraph separation, we tested the following filters:
(1)~a confidence threshold of 0.5, matching the implicit threshold on $P(a|b)$
found in Prop.~\ref{prop_inc_belief}, (2)~a filter based on the Bayesian
factor, and (3)~selecting the rule with the highest confidence for each
conclusion.  In Eq.~\ref{eq_confidence}, we show the estimation of
confidence for a rule $r=a \rightarrow b$, where $\#r$ and $\#a$ are the rule
and premise observations in the dataset respectively.

\begin{equation}
\label{eq_confidence}
\mathrm{Confidence}(r) = \frac{\#r}{\#a}
\end{equation}

The Bayesian factor was estimated using Eq.~\ref{eq_b_factor}, where $\#(a
\rightarrow b')$, $\#r$, and $\#(a \rightarrow x)$ denote how many times the
respective rule  was observed in the timeseries, and $A$ is BRM's final set of
atomic rules.
\begin{equation}
\label{eq_b_factor}
	\begin{aligned}
		\frac{P(b|a)}{P(b'|a)} &= \frac{\#r}{\#(a \rightarrow b')} \\
		\#(a \rightarrow b') &= \sum_{\forall a\rightarrow x \in A \wedge x \ne b  } \#(a \rightarrow x)
	\end{aligned}
\end{equation} 

For comparison, we implemented an exhaustive search FRM and applied the same rule filters to create the graphs. Following previous FRM approaches, e.g. Huang et al.~\cite{Huang2017AssociationsSociodemographicCharacteristicsChemical}, we used a minimum-support threshold of 0.1, and assumed a uniform distribution of all available symbols in the timeseries.

\subsection{Evaluation} 

The goal of this experiment is to separate the symbols into their generating processes. To evaluate BRM performance, we grouped extracted rules into the following categories:
(1)~R\textsubscript{r} contained all possible atomic rules that use $V_r$ symbols, ($|\mathrm{R}_\mathrm{r}|$: $|V_r|^2 = 16$) 
(2)~R\textsubscript{c} contained all atomic, time ordered, decompositions of the chain $10 \rightarrow 11 \rightarrow 12$, i.e., $10 \rightarrow 11$, and $11 \rightarrow 12$,  ($|\mathrm{R}_\mathrm{c}|$: $2$)
(3)~R\textsubscript{rc} contained atomic rules of the form $i \rightarrow j$, where $i \in V_r$ and $j \in V_c$,  ($|\mathrm{R}_\mathrm{rc}|$: $|V_r|*|V_c| = 12$)
(4)~R\textsubscript{cr} contained atomic rules of the form $j \rightarrow i$, where $i \in V_r$ and $j \in V_c$, ($|\mathrm{R}_\mathrm{cr}|$: $|V_r|*|V_c| = 12$), and  
(5)~R\textsubscript{cv} contained atomic rules which were created from all possible pairwise combination of $V_c$ symbols and are not in R\textsubscript{c}, ($|\mathrm{R}_\mathrm{cv}|$: $|V_c|^2 - 2 = 7$). 

We chose the aforementioned rule categories based on the insight they provided
into BRM's functionality. The direct process separation occurred when BRM only
extracted rules from the categories R\textsubscript{r}, R\textsubscript{c}, and
R\textsubscript{cv}. As no rules bind symbols from the two generating process.
Rules in R\textsubscript{cv} were not generated by $p_c(t)$. Thus, they are
considered a separate category. Rules in R\textsubscript{cr} and
R\textsubscript{rc} connect the symbols from $p_c(t)$ and $p_r(t)$ and no direct
process separation was possible. R\textsubscript{cr} and R\textsubscript{rc}
were defined as independent categories to evaluate BRM's effect on symbol
association between frequent and rare symbols, when considering their position
in the rule. To quantify the BRM mining performance, we defined the extraction
rate for a rule category R with size $|\mathrm{R}|$ as shown in
Eq.~\ref{eq_extraction_rate}, where $A$ is the mined rule set.
\begin{equation}
\label{eq_extraction_rate}
\text{Extraction rate} = \frac{|\forall r \in A \cap \mathrm{R}|} {|\mathrm{R}|} * 100 [\%]
\end{equation}

We performed a search over BRM's observation window parameter to analyse
performance. The parametric search looked for observation window sizes in the
range between [2, 500] symbols in one symbol increments. One hundred timeseries
were generated for each observation window size. We selected a window size that minimised the chances of extracting rules from categories R\textsubscript{rc}, R\textsubscript{cr}, and R\textsubscript{cv}.
The selected window size also ensured that BRM always mined all rules from the
R\textsubscript{c} and R\textsubscript{r} categories. 

To evaluate rule filtering methods, we generated a new batch of one hundred 
timeseries. We extracted rules using BRM with the observation window
previously found, and applied the rule filtering methods for each timeseries. 
Finally, we averaged the number of times the symbols were correctly separated into generating
processes $p_r(t)$ and $p_c(t)$ respectively. The same evaluation was performed
using FRM.

\subsection{Results} 
We found that FRM could not retrieve the symbols from $V_c$ as their support was
around 0.01. Using the downward-closure property on support, we inferred that
any atomic rule from $p_c(t)$ will also not pass the minimum-support threshold
of $0.1$. In addition, we saw that if extra symbols were sampled from $p_r(t)$,
then, $p_c(t)$ chain's support diluted. In contrast, BRM's rule selection criterion
does not depend on the number of symbols in the timeseries. Therefore, BRM
always retrieved the chain generated by $p_c(t)$.

Fig.~\ref{fig_ow_sweep} illustrates how the rule categories were extracted as
a function of the observation window size. Using an observation window larger
than the expected symbol timing of $p_c(t)$ of five samples, BRM always
extracted all rules from the R\textsubscript{r} and R\textsubscript{c}
categories, which are needed to separate symbols according to their the
generating processes. Rules in R\textsubscript{cv} were extracted when at least
two partial chains were seen by the observation window. The R\textsubscript{cv}
behaviour for different observation window sizes indicates that, to extract the
original symbol chain from $p_c(t)$, BRM needs to use an observation window in
the range $[5,8]$. Rules from R\textsubscript{rc} always fulfil the increasing
belief criterion and thus, R\textsubscript{rc} rules are always extracted. In
contrast, R\textsubscript{cr} will never pass the increasing belief criterion
and therefore, R\textsubscript{cr} rules are never extracted by BRM, which
implies that a rare premise will never associate with a frequent conclusion.
 
\begin{figure}
	\centering
	\includegraphics[width=\linewidth]{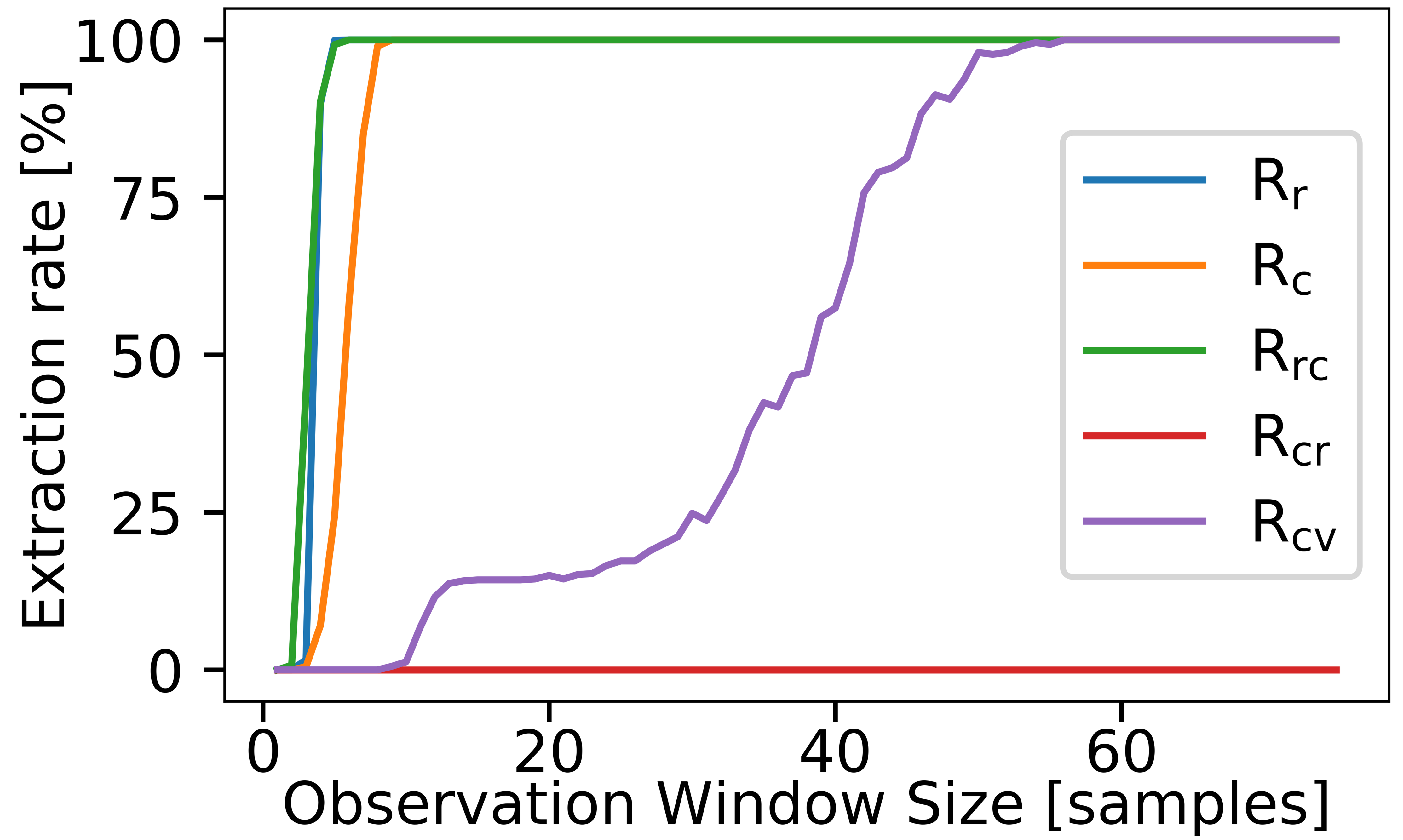}
	\caption{ Search results of observation window sizes for all atomic rule
	categories. BRM extracts all rules from the R\textsubscript{r},
	R\textsubscript{rc}, and R\textsubscript{c} categories when the observation
	window size is greater than the expected value of $p_c(t)$'s symbol timing. 
	As a result of the increasing belief criterion, BRM never extracts rules from the
	R\textsubscript{cr} category. As the observation window grows, BRM extracts rules from R\textsubscript{cv}, which are innocuous for the generating process separation task.
	Therefore, the separation task is achieved by removing R\textsubscript{rc}
	rules from the final set $A$ using a confidence threshold filter.
	} 
	\label{fig_ow_sweep}
\end{figure}

We chose an observation window size of ten symbols to ensure that we will always
extract the rules from R\textsubscript{c}. However, the choice of observation
window meant that we would also retrieve some instances of rules from
R\textsubscript{cv} and R\textsubscript{rc}. Rules in R\textsubscript{cv} had no
impact on process separation and could be ignored. Therefore, the separation
task was achieved by removing the R\textsubscript{rc} rules from the final set
$A$ using the rule filters. We found that the confidence threshold filter always
correctly separated the symbols into generating processes. The Bayesian factor
filter only selected the chain rules and therefore no symbols from $p_r(t)$ are
grouped. Finally, the best confidence per conclusion filter failed to separate
symbols into two generating processes, because $p_c(t)$ symbol 10 is always
associated as conclusion with a $p_r(t)$ symbol.

%% file: data_base_mining.tex
\section{Census and Chemical Exposure Database Mining}
In this experiment, we illustrate BRM for database mining and analyse the
differences in results retrieved by BRM and FRM.

\subsection{Methodology}
We used the publicly available dataset from Huang et
al.~\cite{Huang2017AssociationsSociodemographicCharacteristicsChemical}. The
dataset comprises US census tract information from the American Community Survey
(ACS) 5-year summary files for the 2010 to 2014 period. Moreover, the dataset
contained chemical exposure data generated from the 2011 National-Scale Air
Toxics Assessment (NATA), specifically, air pollutant exposure concentration.
Huang et al. reported results for two mining scenarios: (1) mining rules using
the socioeconomic variables as premises and chemical exposure variables as
conclusions~(S$\rightarrow$C), and (2) mining rules within the socioeconomic
dataset~(S$\rightarrow$S). Huang et al. used minimum-support of $0.1$ and
lift~$>1$  as thresholds for FRM-based rule selection. We replicated their
findings with an exhaustive search FRM algorithm and compared the extracted
rules to BRM results.

We followed Huang et al.'s categorisation of variables in the dataset.
Socioeconomic scores were divided into deciles, chemical variables into
quartiles, and age group used the ranges: [0--20], (20--30], (30--35], (35--38],
(38--40], (40--50], (50--150]. The poverty score was computed as the percentage
of  population per track that had a ratio between income and poverty level below
1.5. Additionally, deciles seven through ten were combined. Education score was
calculated as the population percentage per track that had demonstrated
education beyond high school level. Deciles eight through ten of the education
score where merged. Finally, the race score was calculated as the percentage of
the non-white population per track.

\subsection{Evaluation}
We chose Huang~et
al.~\cite{Huang2017AssociationsSociodemographicCharacteristicsChemical} as a
reference because they analysed the problem and provided relevance criteria to
interpret the extracted rules in their field.  We reasoned that if BRM extracted
rules that were in Huang~et al.'s selection, then, BRM's increasing belief
approach can extract relevant rules.

In our preliminary test, we noticed that BRM was very selective and only
extracted five rules. Therefore, we used the Selector $s$ to control BRM's
sensitivity~(see Eq.~\ref{eq_selector}). We performed a parametric search on the
Selector $s$, using 200 equidistant samples from the range $[0, 1]$,  to find
the value that allowed BRM to extract the missing rules from Huang~et al.'s
S$\rightarrow$C and S$\rightarrow$S scenarios.

For any new-found rule, we computed the odds-ratio~(OR), with a 95\% confidence
interval~(CI). The CI was calculated using 10000 runs of bootstrapping,
following Huang~et
al.~\cite{Huang2017AssociationsSociodemographicCharacteristicsChemical}'s
evaluation of rule relevance.

\subsection{Results}
Table~\ref{tab_mined_rules_1} lists the mined rules from Huang et
al.~\cite{Huang2017AssociationsSociodemographicCharacteristicsChemical}
S$\rightarrow$C scenario, with their respective BRM's selector value~$s$,
support, and probability of premise given the conclusion~$P(a|b)$. Using BRM's
default behaviour, i.e.,  with $s=1$, we extracted only five rules. To extract
all rules from Huang et al., BRM required a selector lower than $s=0.68$.

\begin{table}[h!]
	\caption{BRM's extraction of socioeconomic and chemical exposure association
	rules~(S$\rightarrow$C). The table shows required selector values to extract
	Huang et
	al.~\cite{Huang2017AssociationsSociodemographicCharacteristicsChemical}'s set
	of rules.} \label{tab_mined_rules_1} \centering
	\begin{tabular}{lccc}
		\toprule
		{Rules} &  \multicolumn{1}{c}{Selector\,$s$}&  Support &    $P(a|b)$ \\
		\midrule
		Race score = 1 $\rightarrow$ Diesel = Q1           &     1.000 &    0.144 &   0.578 \\
		Race score = 1 $\rightarrow$ Butadiene = Q1        &     1.000 &    0.142 &   0.585 \\
		Race score = 1 $\rightarrow$ Toluene = Q1          &     1.000 &    0.138 &   0.566 \\
		Race score = 1 $\rightarrow$ Benzene = Q1          &     1.000 &    0.130 &   0.535 \\
		Race score = 1 $\rightarrow$ Acetaldehyde = Q1     &     1.000 &    0.126 &   0.517 \\
		Age group = 40--50 $\rightarrow$ Diesel = Q1       &     0.995 &    0.124 &   0.499 \\
		Age group = 40--50 $\rightarrow$ Butadiene = Q1    &     0.910 &    0.116 &   0.477 \\
		Age group = 40--50 $\rightarrow$ Toluene = Q1      &     0.874 &    0.114 &   0.467 \\
		Age group = 40--50 $\rightarrow$ Benzene = Q1      &     0.849 &    0.112 &   0.459 \\
		Age group = 40--50 $\rightarrow$ Acetaldehyde = Q1 &     0.769 &    0.106 &   0.436 \\
		Race score = 1 $\rightarrow$ Cyanide = Q3          &     0.754 &    0.110 &   0.431 \\
		Race score = 1 $\rightarrow$ Toluene = Q2          &     0.698 &    0.104 &   0.411 \\
		Race score = 1 $\rightarrow$ Diesel = Q2           &     0.683 &    0.102 &   0.408 \\
		\bottomrule
	\end{tabular}
\end{table}

Tab.~\ref{tab_mined_rules_2} lists mined association rules in the
S$\rightarrow$S scenario. BRM found two out of the six rules reported by
Huang~et al. Additionally, BRM found three rules, highlighted in
Tab.~\ref{tab_mined_rules_2}, that did not pass the minimum-support and lift
criteria from Huang et al. To retrieve all six rules from Huang~et al., BRM
required a selector $s=0.48$ and extracted 41 additional rules.

\begin{table}
	\caption{In the S$\rightarrow$S scenario,  BRM's extracted rules from the
	socioeconomic dataset. Rules in bold did not pass Huang~et 
	al.\cite{Huang2017AssociationsSociodemographicCharacteristicsChemical}
	minimum-support criteria. We list the Selector $s$ values required by BRM to
	extract Huang~et al.'s rules.}
	\label{tab_mined_rules_2}
	\centering
	\begin{tabular}{lccc}	
		\toprule
		{Rules} &  Selector\,$s$ &  Support &  $P(a|b)$ \\
		\midrule
		Race score = 1 $\rightarrow$ Age group = 40--50     &     1.000 &    0.172 &   0.516 \\
		Age group = 40--50 $\rightarrow$ Race score = 1     &     1.000 &    0.172 &   0.528 \\
		Race score = 1 $\rightarrow$ Poverty score = 2      &     0.764 &    0.111 &   0.434 \\
		Poverty score = 2 $\rightarrow$ Race score = 1      &     0.513 &    0.111 &   0.340 \\
		Age group = 40--50 $\rightarrow$ Poverty score = 2  &     0.749 &    0.110 &   0.430 \\
		Poverty score = 2 $\rightarrow$ Age group = 40--50  &     0.487 &    0.110 &   0.329 \\
		\textbf{Poverty score = 1 $\rightarrow$ Education score = 8} &     1.000 &    0.038 &   0.623 \\
		\textbf{Poverty score = 1 $\rightarrow$ Education score = 7} &     1.000 &    0.034 &   0.510 \\
		\textbf{Race score = 1 $\rightarrow$ Age group = 50--150}    &     1.000 &    0.015 &   0.725 \\
		\bottomrule
	\end{tabular}
\end{table}  

The conditional probability $P(a|b)$, i.e., the premise $a$ probability given
that the conclusion $b$ was observed, describes the ratio of conclusion
observations associated to the premise. Thus, a higher $P(a|b)$ reflects the
rule's relevance, because a larger portion of the conclusions can  be explained
by the premise. The rules extracted exclusively by BRM show that the conclusions
are associated between 51\% to 73\% with their respective premises in the
S$\rightarrow$S scenario. Whereas for FRM based rules, the maximum association
is 52\% for the S$\rightarrow$S scenario. Therefore, BRM rules would have a
higher chance of correctly predicting the conclusion when the premise is
observed. In the S$\rightarrow$C scenario, the best association is 59\%  for
both BRM and FRM.

Fig.~\ref{fig_selector_sweep} illustrates the selector's effect on the total
number of extracted rules for both database mining scenarios. The
Selector $s$ had an exponential effect on the exclusion of rules in both scenarios.
The S$\rightarrow$C scenario converged to the Huang et al.'s rules,
while the S$\rightarrow$S scenario always extracted additional rules. 

\begin{figure}
	\centering
	\includegraphics[width=\columnwidth]{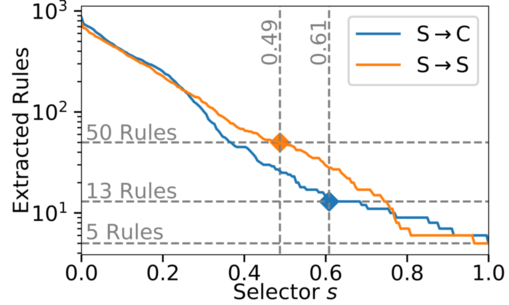}
	\caption{ Selector $s$ value sweep for two rule mining scenarios: the
	socioeconomic and chemical exposure~(S$\rightarrow$C) datasets, and within the
	socioeconomic dataset~(S$\rightarrow$S). Five rules were extracted with BRM's
	default setup, $s=1$. The marker \ding{117} illustrates the minimum Selector
	$s$ required to extract all rules reported by Huang et
	al.~\cite{Huang2017AssociationsSociodemographicCharacteristicsChemical}. In
	S$\rightarrow$C and for values of $s \ge 0.61$, BRM extracts rules that are in
	the FRM rule set. In S$\rightarrow$S, the minimum required value
	of $s= 0.49$ caused BRM to extract 40 additional rules.}
	\label{fig_selector_sweep}
\end{figure}

With $s=1$ there were no conjunctive premise candidates in ether scenarios. In
S$\rightarrow$C, even after setting $s=0.6$, there were no resulting conjunctive
rules that pass the increasing belief criterion. In S$\rightarrow$S, the
disjunctive conclusion rules Poverty score=1$\rightarrow$ Education score = (7
or 8), and Race score = 1$\rightarrow$Age group = 40--150, both have increasing
belief with $P(a|b)$ of 56\% and 53\% respectively. Therefore, rules with
disjunctive conclusion suggest that the categories from Age group and Education
score could be merged.

Tab.~\ref{tab_mined_rules_3} shows the odds' ratio~(OR) and estimated 95\%
confidence interval~(CI) for rules exclusively extracted by BRM in the
S$\rightarrow$S scenario. The OR analysis showed that the new rules had higher
OR than the rules from Huang et al, whose OR ranged from 1.75 to 3.56. Rules
with increasing belief had the largest OR in both mining scenarios. Thus, BRM
rules are more likely to appear in a repeat experiment, and therefore, BRM
rules may be deemed more desirable.
\begin{table}
	\caption{Odds-ratio~(OR) and confidence interval~(CI) for BRM extracted rules
	in the S$\rightarrow$S scenario. The 95\% CI was estimated using 10000
	bootstrapping iterations.}
	\label{tab_mined_rules_3}
	\centering
	\setlength{\tabcolsep}{4pt}
	\begin{tabular}{lrrr}
		\toprule
		\multicolumn{1}{c}{Rule} & \multicolumn{1}{c}{OR} & \multicolumn{2}{c}{Est. 95\% CI} \\
		\midrule
		Race score = 1 $\rightarrow$ Age group = 40--50 & 3.56 & 3.45 & 3.68 \\
		Age group = 40--50 $\rightarrow$ Race score = 1 & 3.56 & 3.45 & 3.68 \\
		Poverty score = 1$\rightarrow$ Education score = 8 & 11.18 & 10.49 & 11.94 \\
		Poverty score = 1$\rightarrow$ Education score = 7 & 6.74 & 6.35 & 7.17 \\
		Race score = 1$\rightarrow$ Age group = 50--150 & 5.68 & 5.10 & 6.39 \\	
		\bottomrule
	\end{tabular}
\end{table}

%% file: time_series_mining.tex
\section{Rehabilitation Routine Mining} 
We show how BRM can be used to interpret
patient behaviour during stays at a day care rehabilitation centre. We compare
BRM and FRM for classifying patients into physically active and sedentary
groups. However, we observed that FRM represents the cohort's average behaviour
and thus fails to assign patients to groups.

\begin{figure*}[t]
	\centering
	\setlength{\tabcolsep}{1pt}
	\begin{tabular}{|>{\centering\arraybackslash}m{1em} %
			|>{\vspace {4pt}\centering\arraybackslash}m{0.23\textwidth}%
			|>{\vspace {4pt}\centering\arraybackslash}m{0.23\textwidth}%
			|>{\vspace {4pt}\centering\arraybackslash}m{0.23\textwidth}%
			|>{\vspace {4pt}\centering\arraybackslash}m{0.23\textwidth}|}
		\hline
		{} & No Filter & Bayesian Factor  & Confidence $\geq$ 0.5 & Best Confidence 
		\\
		\hline
		\rotatebox[origin=m]{90}{BRM} & 
		\Centerstack{\includegraphics[width=0.23\textwidth]{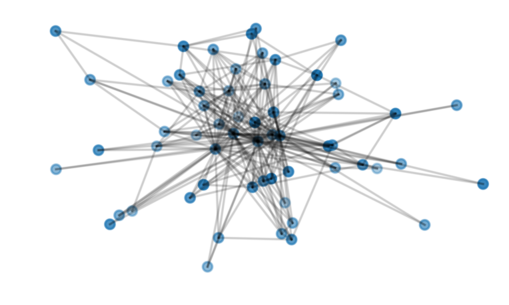}} &
		\Centerstack{\includegraphics[width=0.23\textwidth]{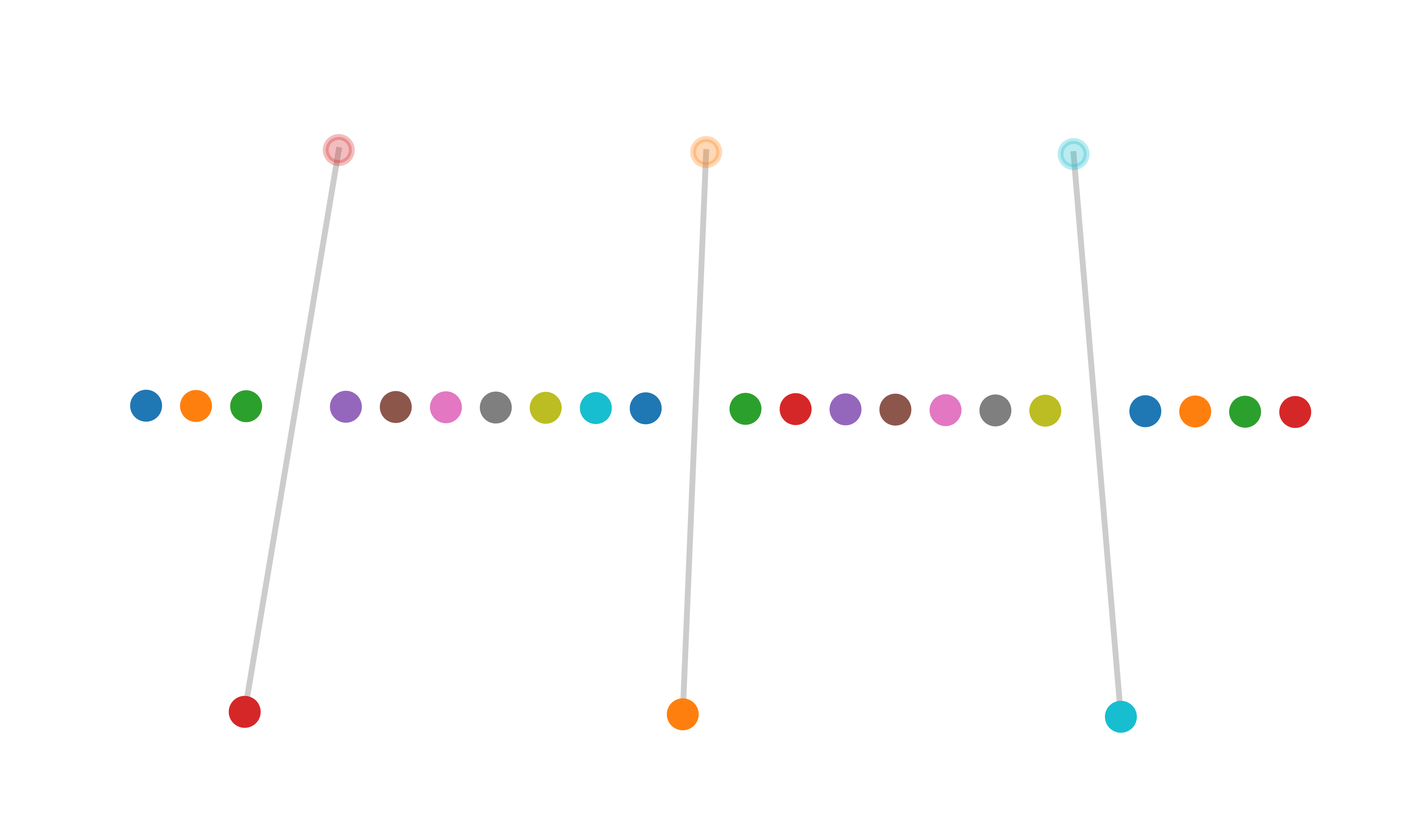}} &
		\Centerstack{\includegraphics[width=0.23\textwidth]{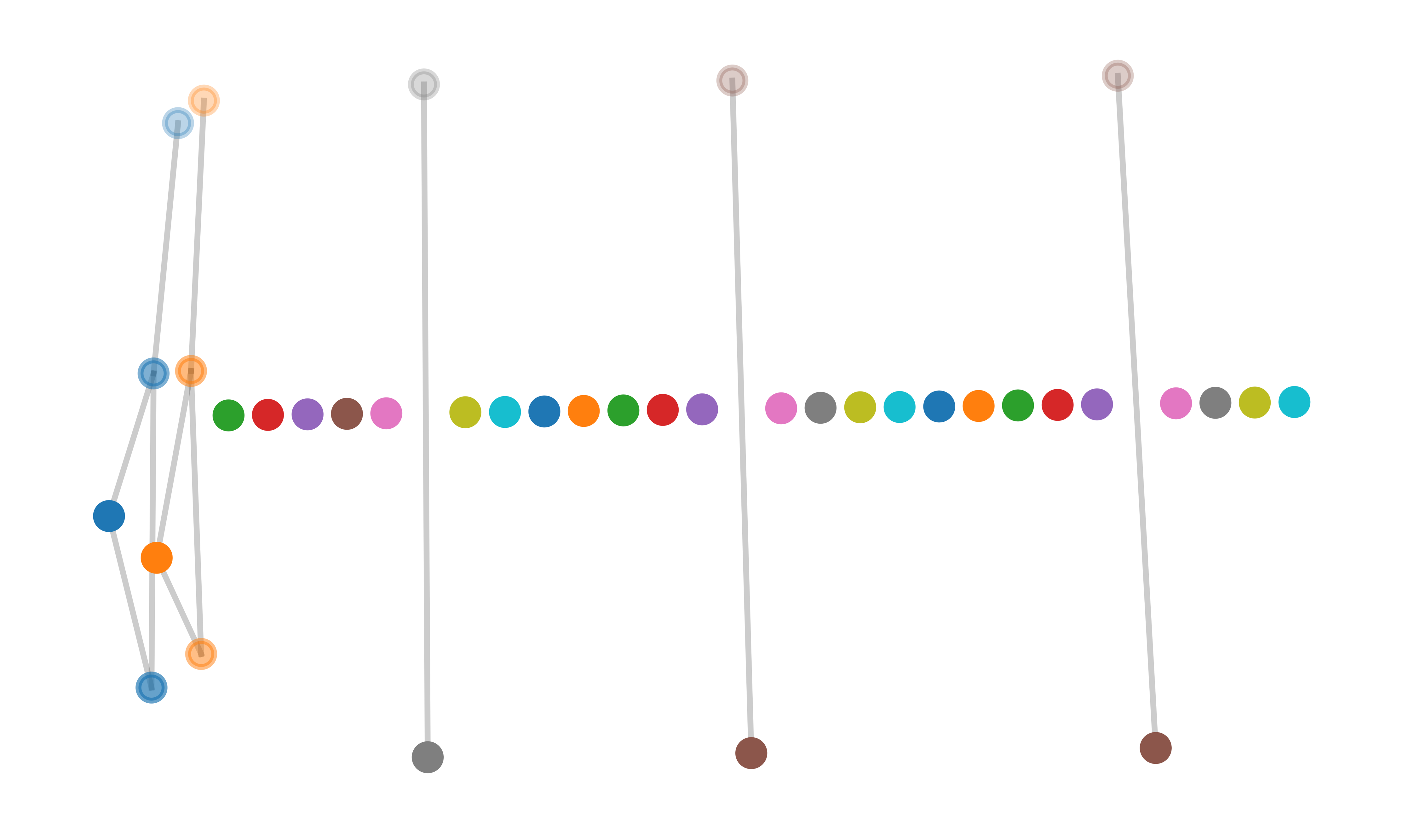}} & 
		\Centerstack{\includegraphics[width=0.23\textwidth]{figures/bre_best-c_graph}} \\
		\hline
		\rotatebox[origin=m]{90}{FRM} &
		\Centerstack{\includegraphics[width=0.23\textwidth]{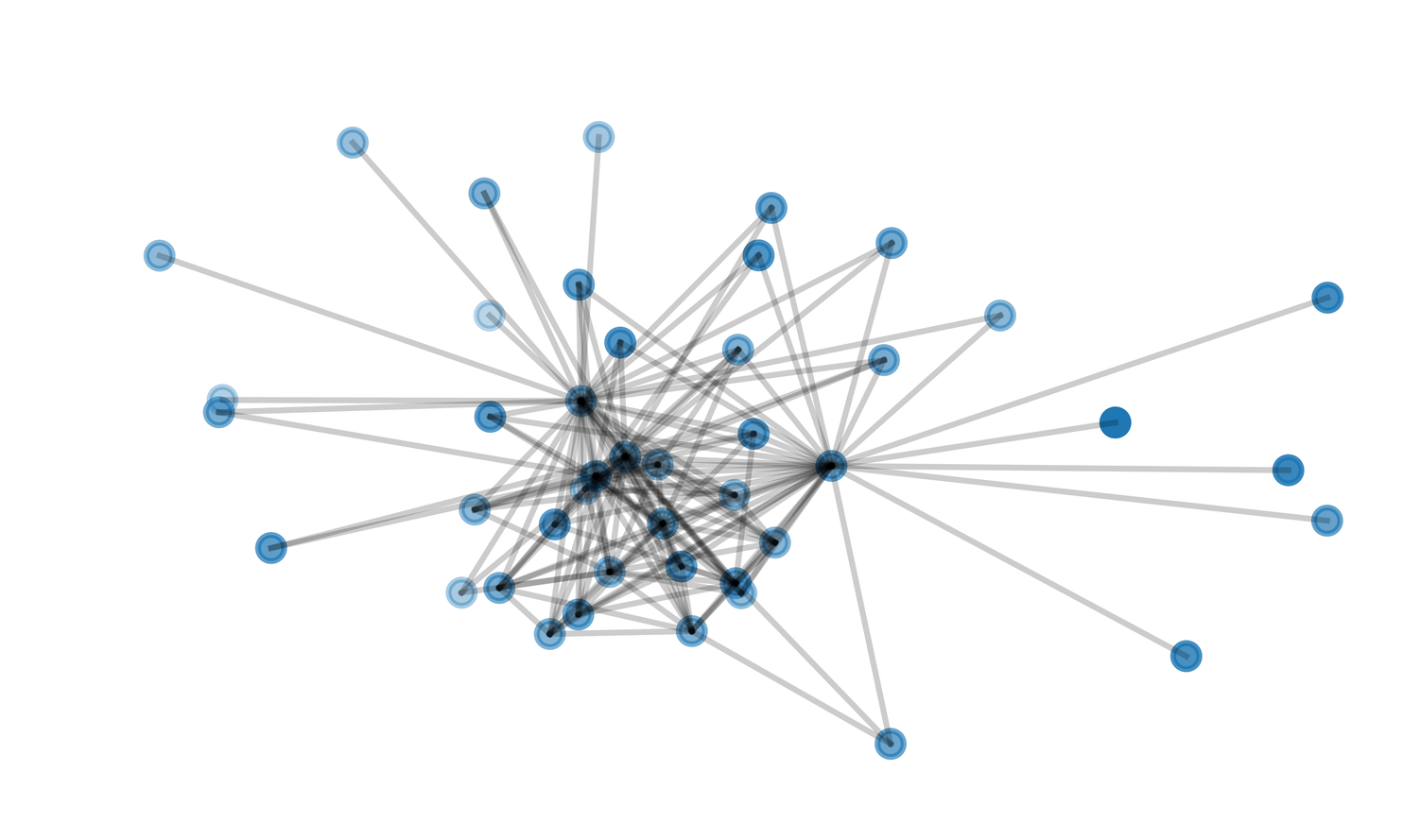}} & %
		\Centerstack{\includegraphics[width=0.23\textwidth]{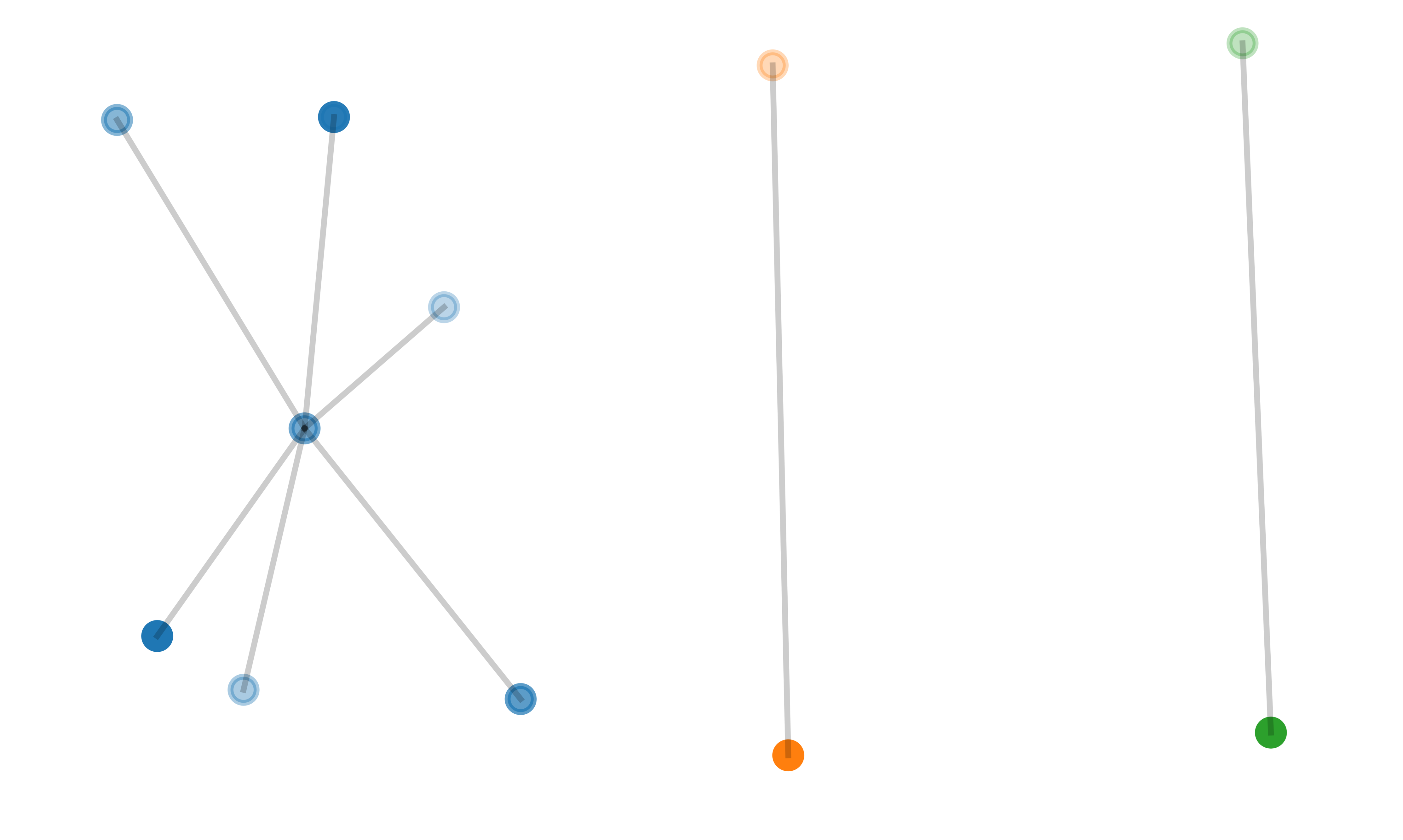}} & %
		\Centerstack{\includegraphics[width=0.23\textwidth]{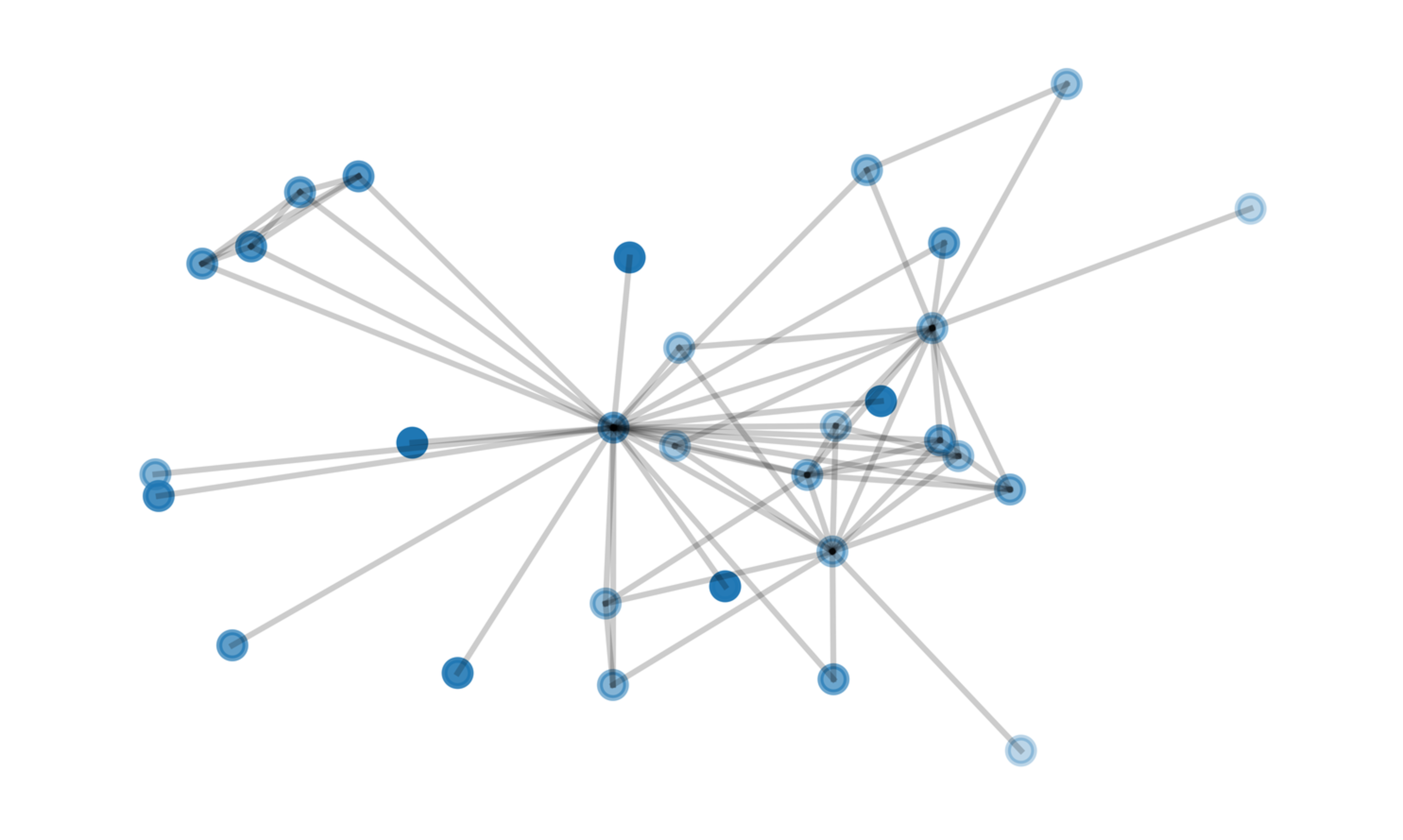}} & %
		\Centerstack{\includegraphics[width=0.23\textwidth]{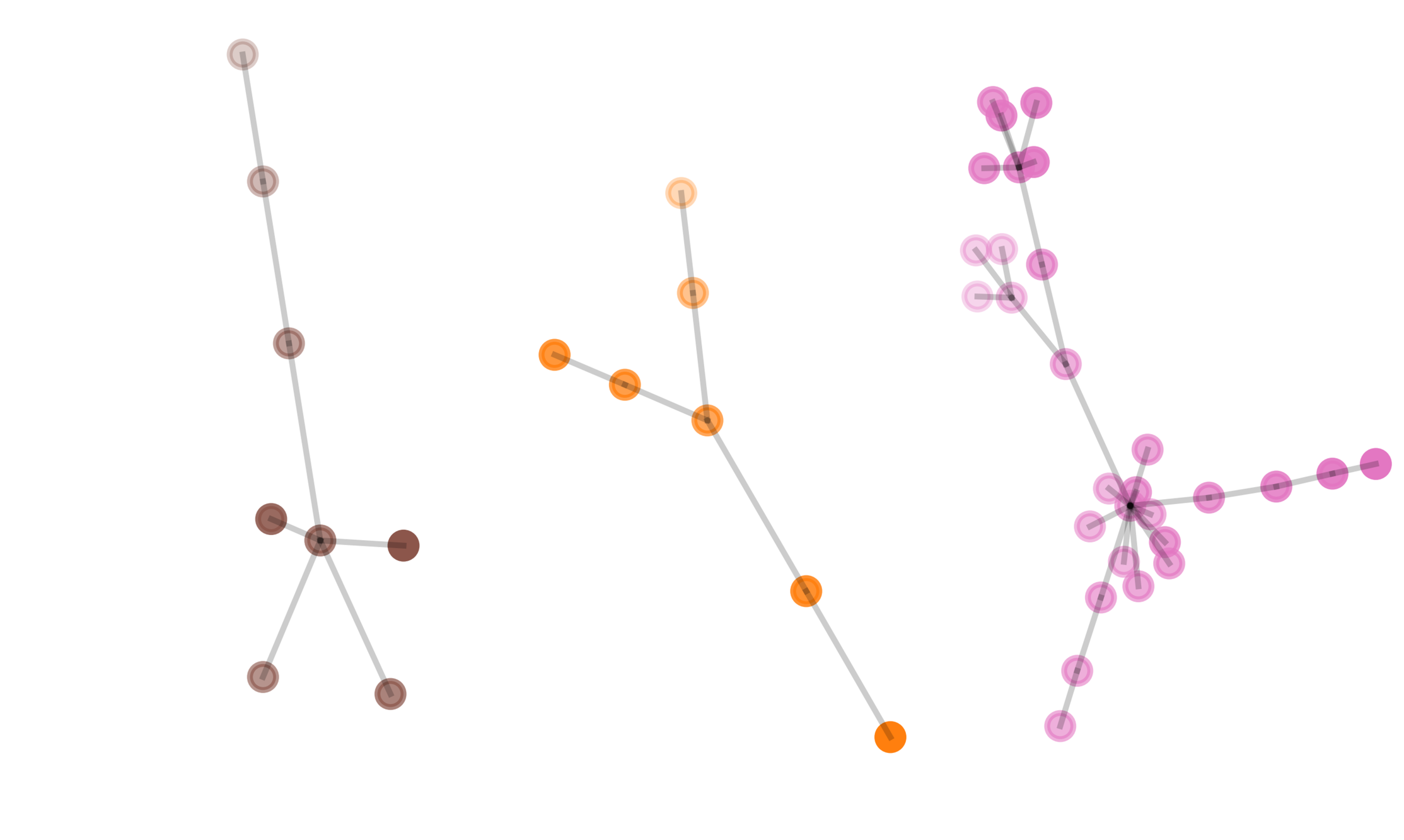}}   \\
		\hline
	\end{tabular}
	\caption{Graphs constructed using rules derived by BRM, FRM, and different rule
		filters. Without a filter, both BRM and FRM produce a single graph and no
		useful routine information is extracted. A balance between activity count per
		graph and number of independent graphs was achieved using the \emph{Best
			confidence} per conclusion filter.} \label{fig_rule_selection}
\end{figure*}

\subsection{Methodology}
For the rehabilitation routine mining experiment, we used activity labels from
the longitudinal stroke rehabilitation study of Derungs et
al.~\cite{Derungs2018LongitudinalWalkingAnalysisHemiparetic}. The study was
approved of the Swiss cantonal Ethics committee of the canton Aargau,
Switzerland (Application number: 2013/009). There were eleven patients in the
study, aged 34 to 75 years, five female, and four used a wheelchair. In
addition, we used data from a patient excluded from the original rehabilitation
study~\cite{Derungs2018LongitudinalWalkingAnalysisHemiparetic}, for a total of
12 patients.

Patients visited the day care centre for approximately three days per week over
three months to participate in individual and group training sessions, socialise
with others, and follow personal activity preferences. Some training sessions
available to patients were physiotherapy, ergotherapy, and training in the gym.
Patients performed activities of daily living, including walking, eating and
drinking, setting the table, writing, and making coffee. Behaviour of each
patient was recorded for up to eight hours on 10 days at the centre by two
observers accompanying patients. In addition, body motion was recorded using
inertial sensors attached to wrists, upper arms, and tight positions. During the
observation time, the examiners annotated patient activities using a customised
annotation tool on a smartphone, resulting in a total of 16226 activity labels.
Therapists scored patients for their ability to execute activities of daily
living independently using the Extended Barthel
Index~(EBI)~\cite{Prosiegel1996ErweiterteBarthelIndexEBIeineNeue}. The EBI
consists of 16 categories. Each category receives a score within the range zero
to four, where zero means that the patient requires full support, and four means
the patient can live independently.

We started by mining atomic rules from patient's activity labels. Subsequently,
we filtered the rules and created graphs. Each resulting independent subgraph
was considered a routine, which was analysed by a study observer to assign a
routine label. We observed that by removing an active patient from the mining
dataset, the resulting routines would change with respect to the routines
extracted using all patients. Removing a sedentary patient had no effect in the
resulting routines. Therefore, we classified patients into physically active or
sedentary groups by removing the patient's data from the mining dataset. We
refer to the patient classification method as patient exclusion process~(PEP).

\subsection{Evaluation} For BRM, we used the default Selector $s=1$ value and an
observation window of 20 minutes. For FRM, we used an exhaustive search
algorithm with a minimum-support threshold of 0.0038. The minimum-support
threshold value was chosen to obtain the same number of rules as for BRM. Both
methods were set to extract atomic rules only, used label names as symbols,
and their start times as symbol timestamp.

We evaluated each mining method by submitting their extracted rules to the same
post-processing two stage procedure: (1) rule filtering, and (2) graph-based 
routine classification. In the rule filtering stage, we evaluated three filtering 
methods: Bayesian factor 
(Eq.~\ref{eq_b_factor}), confidence threshold of 0.5 and best confidence per
conclusion.  With the retained rules, we built a graph and extracted routines as
independent subgraphs. For each mining method, we chose a filter for the
post-processing procedure that provided a balance between activity label count
per graph and the number of independent graphs. Subsequently, we compared the
mining methods with their respective post-processing procedures to classify
patients into physically active, and sedentary groups using PEP.

Based on the type of the majority of activities in the routine, a study observer
named BRM routines as socialising, eating, using the phone, intense and balance
training. Whereas, FRM routines were named mobility, eating, and cognitive-motor
training. We observed that FRM routines lacked emphasis on activities related to
socialising.

\subsection{Results}

The \emph{No Filter} column in Fig.~\ref{fig_rule_selection} shows the
resulting graphs based on atomic rules extracted by BRM and FRM methods.
Activities in both graphs are hyperconnected, i.e., multiple edges connect
activities. However, for FRM, there are nodes with one edge. FRM rules do not
describe the flow from one activity to another, but rather, the associations of
repeating events, e.g., repetitions of an exercise. In contrast, BRM looks for
successive activities, and the respective low count of activity transitions vs
exercise repetitions does not affect the rule selection. For both mining
algorithms, the hyperconnected graph yielded no useful routine information.

We evaluated the effect of rule filtering on graph creation.  With a Bayesian
factor $\geq 1$, BRM mined rules focus mostly on self-referencing activities,
e.g., $walking \rightarrow walking$, resulting in single activity
subgraphs. In contrast, the Bayesian factor filter removed most of the FRM
rules. The resulting subgraphs had too few activities to consider them as
routines. For FRM rules, the confidence filter reduced the graph size, but it
was unable to create independent subgraphs. However, with BRM rules, the
confidence filter selected many self-referencing rules. The confidence filter
was able to create two more subgraphs than the Bayesian factor, 
containing four activities each. However, there were too many single activity
subgraphs to consider the split as routines. We obtained the best balance
between the number of subgraphs and activities per graphs using the best
confidence per conclusion filter. After filtering, BRM-mined rules yielded five
routines, whereas FRM-mined rules yielded only three. Fig.~\ref{fig_rule_selection}
illustrates the resulting subgraphs for each mining algorithm and rule filtering
method.

\begin{figure*}
	\centering
	\begin{tabular}{|>{\centering\arraybackslash}m{1em} %
			|>{\vspace {4pt}\centering\arraybackslash}m{0.23\textwidth}%
			|>{\vspace {4pt}\centering\arraybackslash}m{0.23\textwidth}%
			|>{\vspace {4pt}\centering\arraybackslash}m{0.23\textwidth}%
			|>{\vspace {4pt}\centering\arraybackslash}m{0.23\textwidth}|}
		\multicolumn{5}{c}{\includegraphics[width=\textwidth]{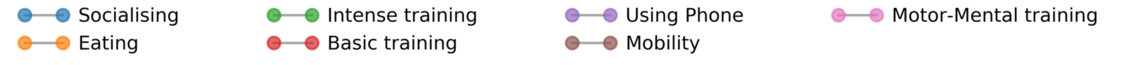}}\\
		\hline
		{} & All Patients & Without Patient ID 9  & Without Patient ID 10 & Without Patient ID 12 \\
		\hline
		\rotatebox[origin=m]{90}{BRM} & 
		\Centerstack{\includegraphics[width=0.23\textwidth]{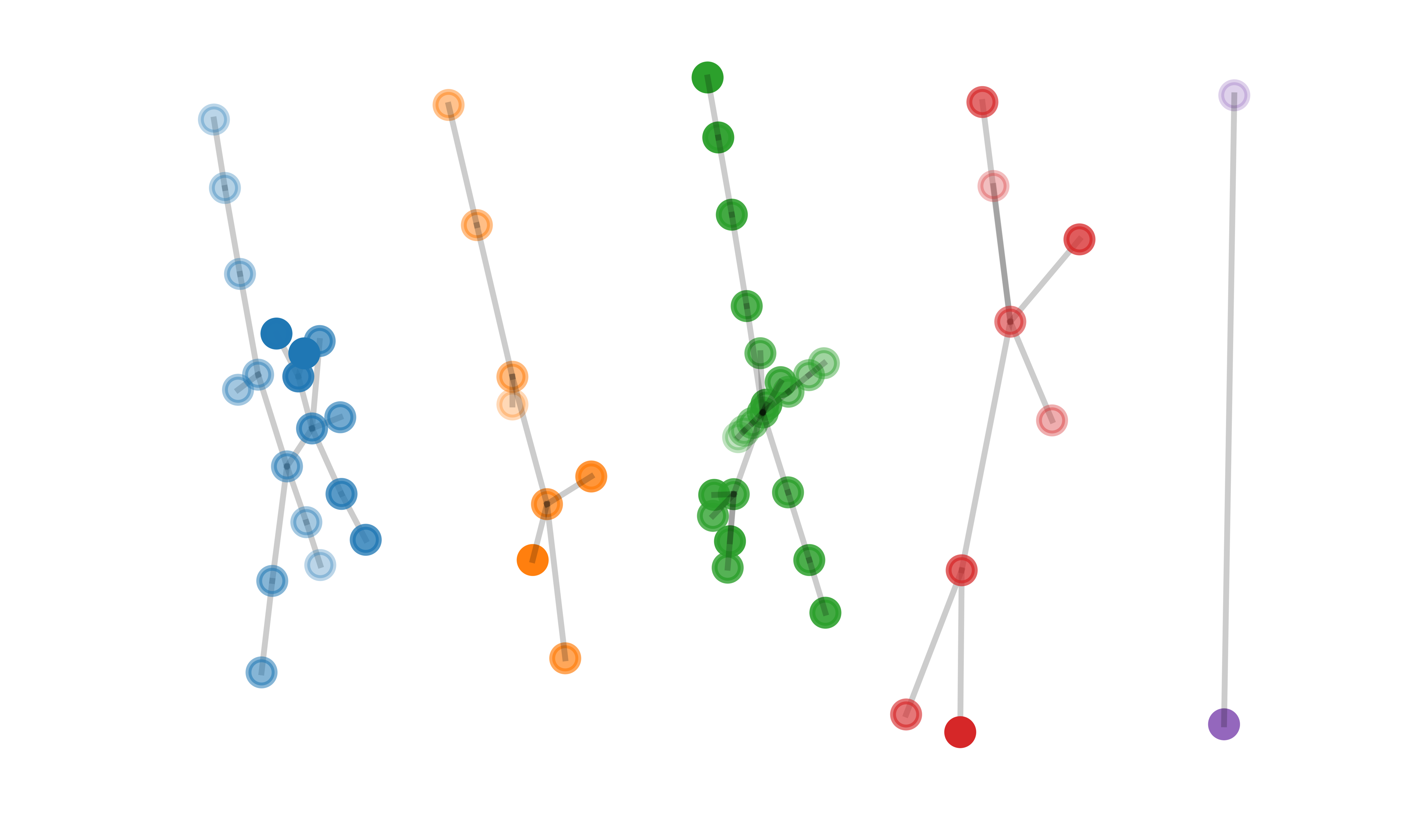}} &
		\Centerstack{\includegraphics[width=0.23\textwidth]{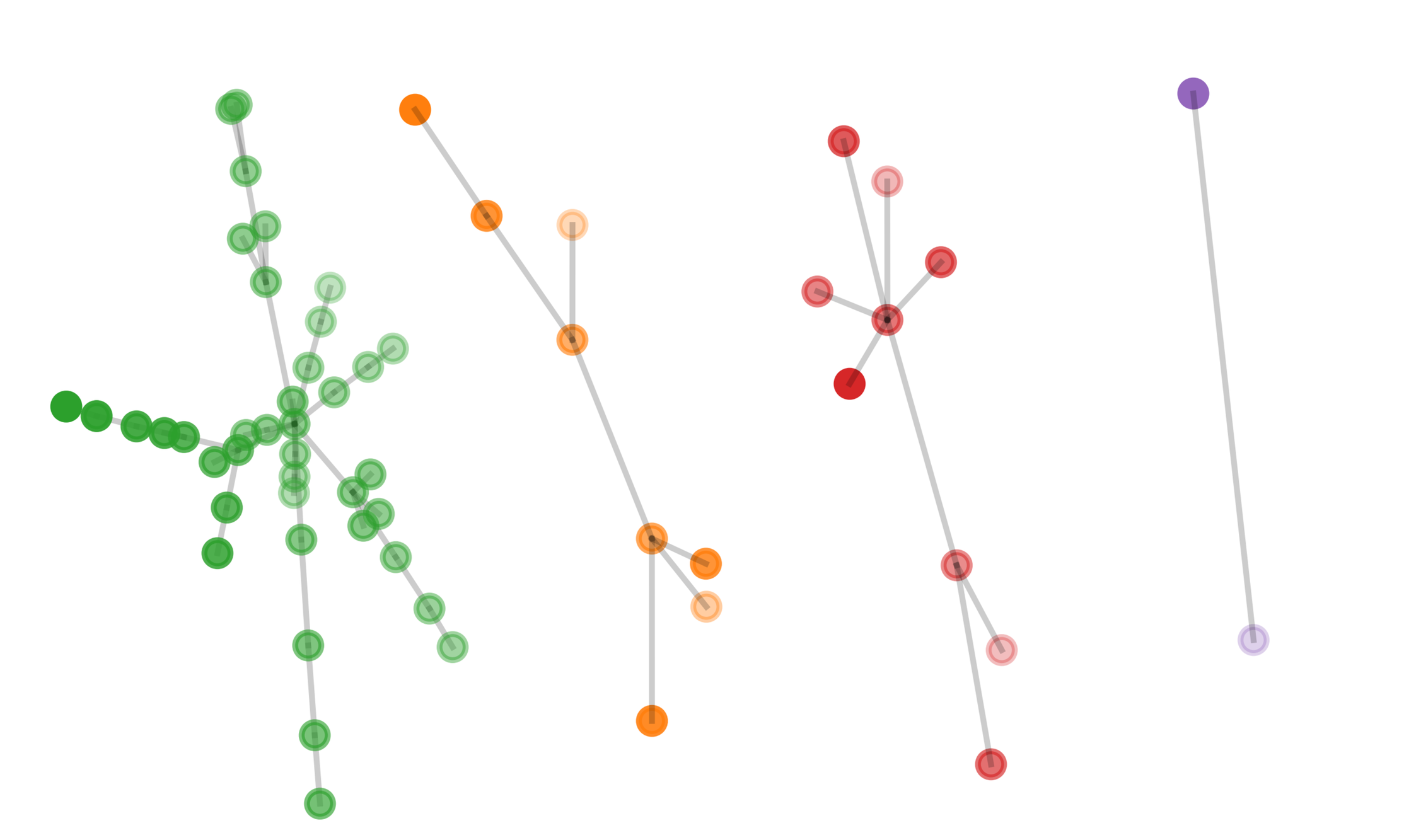}} &
		\Centerstack{\includegraphics[width=0.23\textwidth]{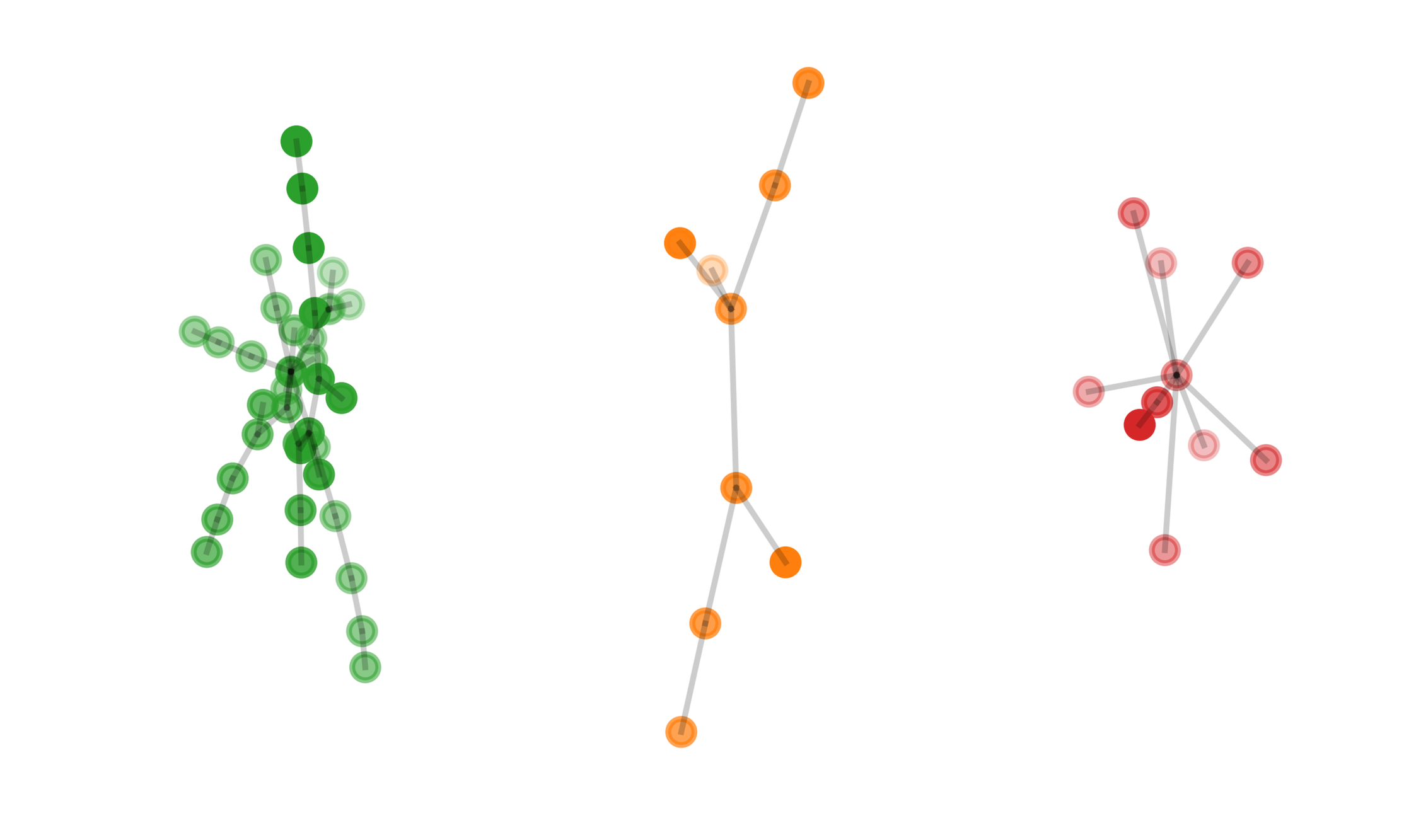}} & 
		\Centerstack{\includegraphics[width=0.23\textwidth]{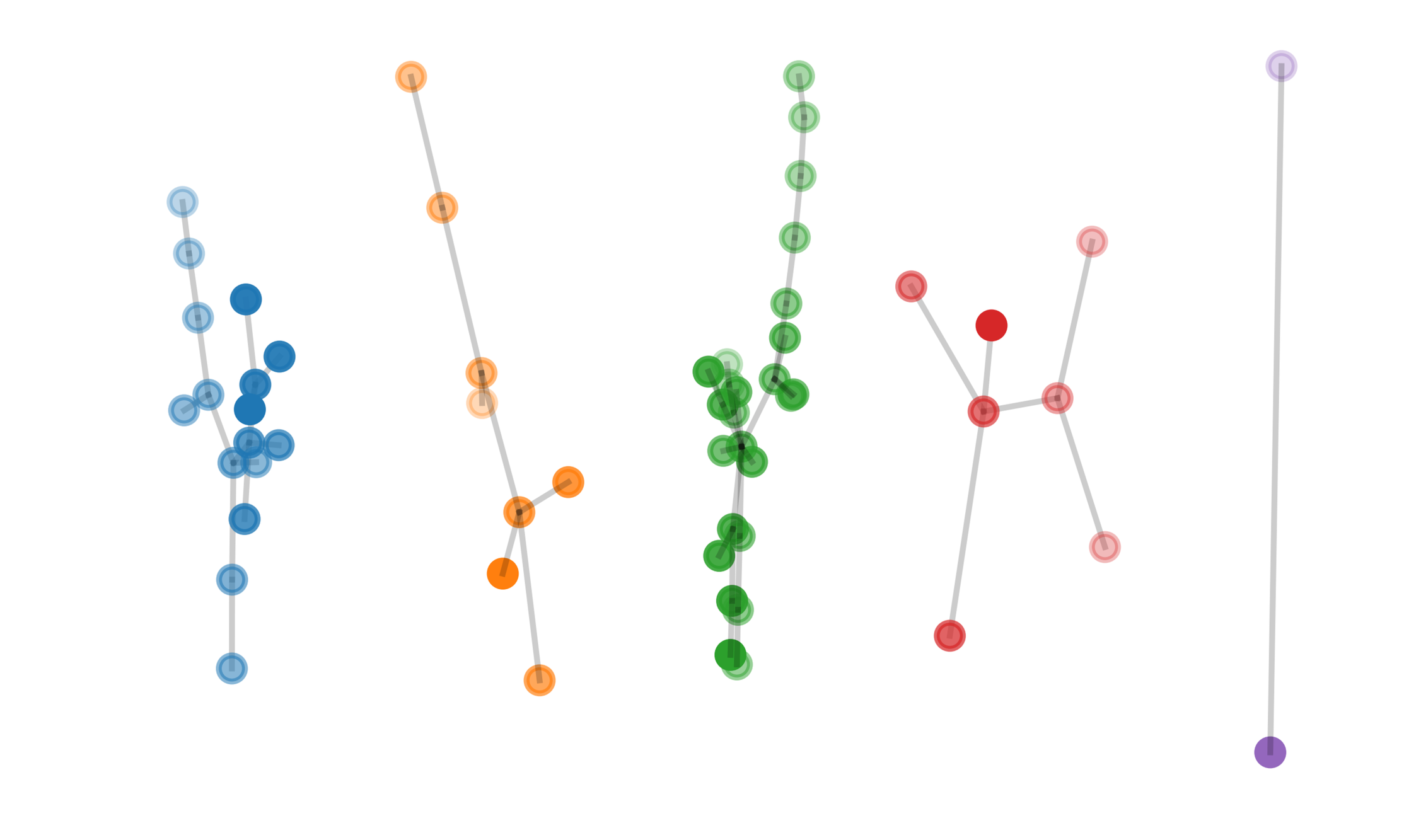}} \\
		\hline
		\rotatebox[origin=m]{90}{FRM} &
		\Centerstack{\includegraphics[width=0.23\textwidth]{figures/frm_best-c_graph}} & %
		\Centerstack{\includegraphics[width=0.23\textwidth]{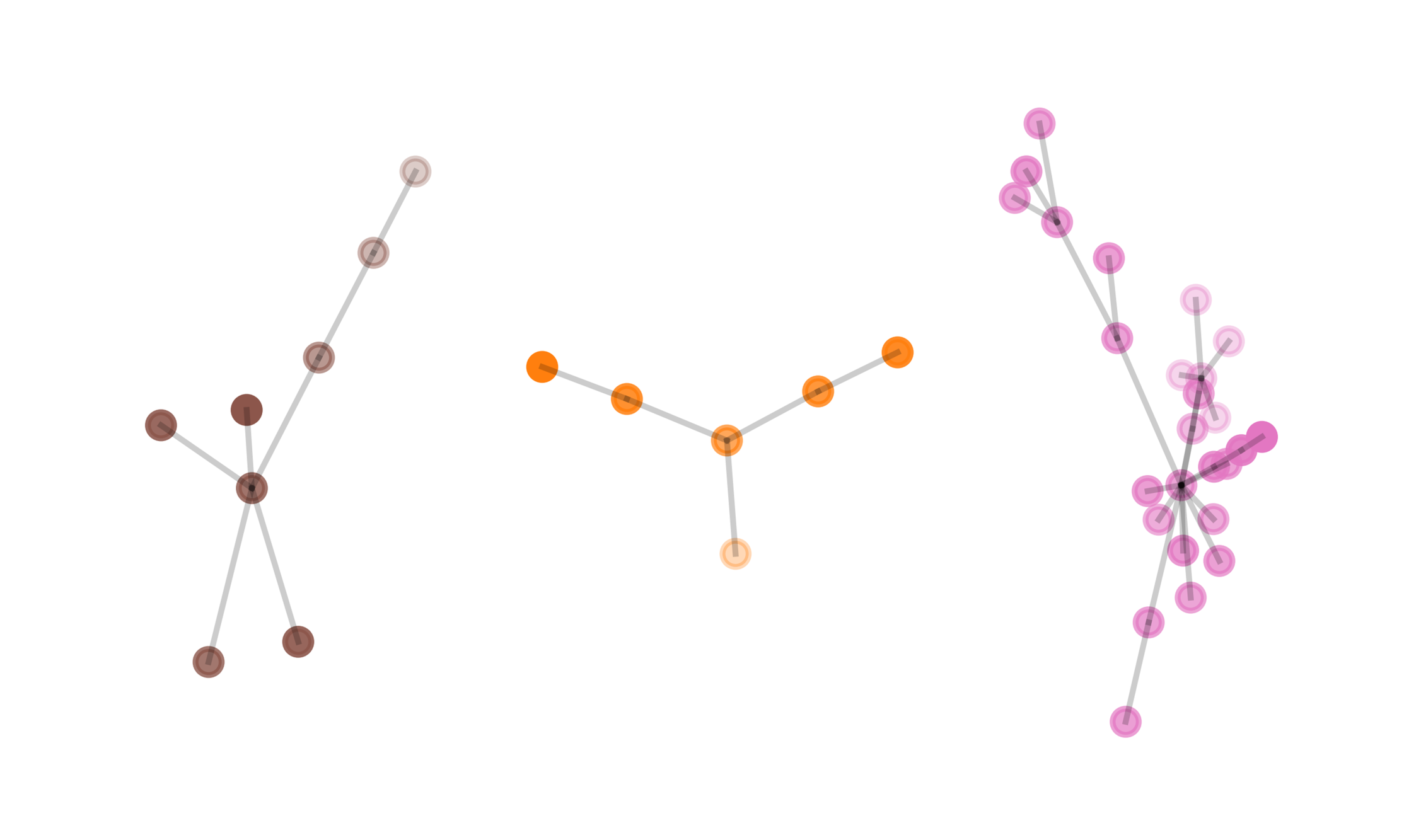}} &
		\Centerstack{\includegraphics[width=0.23\textwidth]{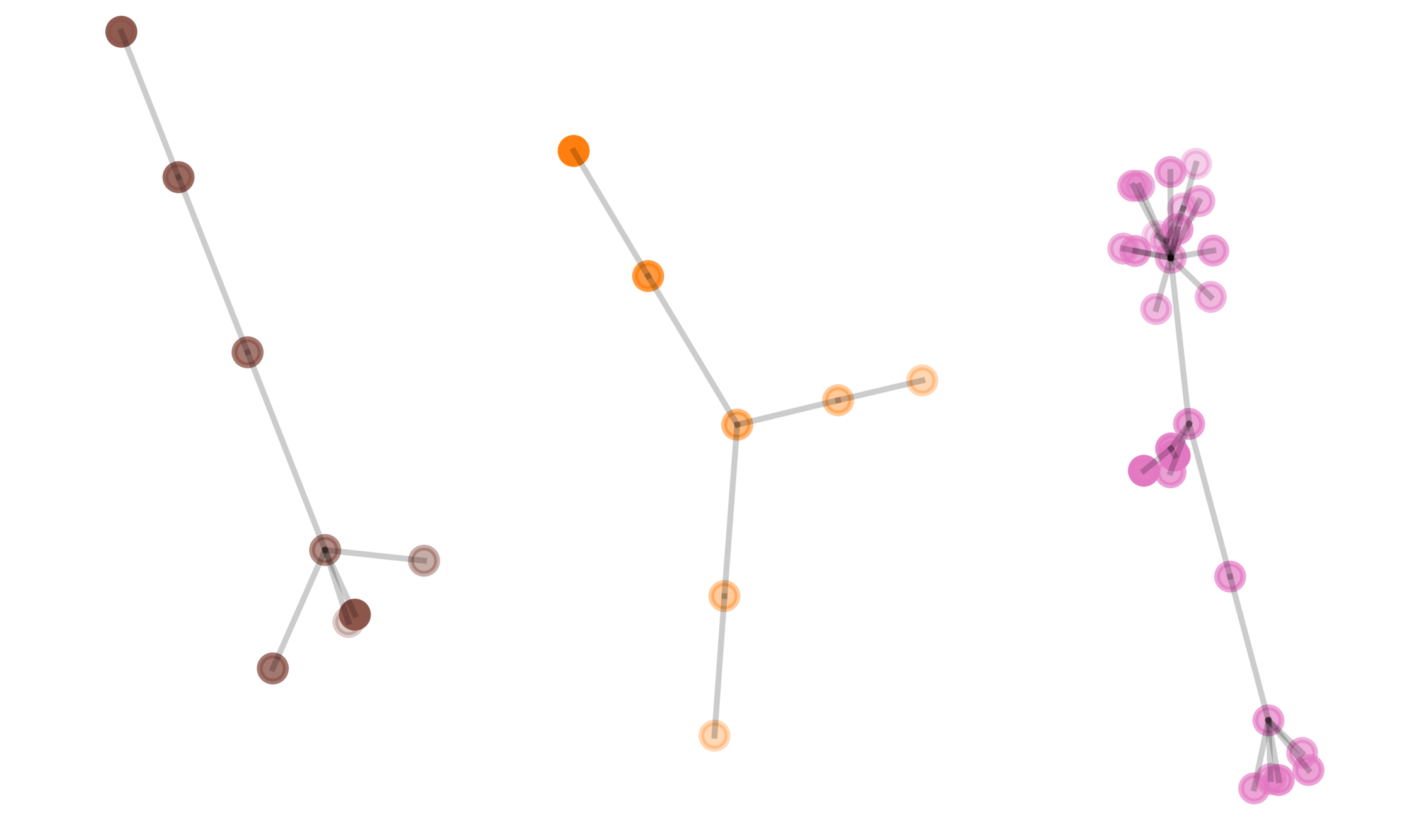}} & 
		\Centerstack{\includegraphics[width=0.23\textwidth]{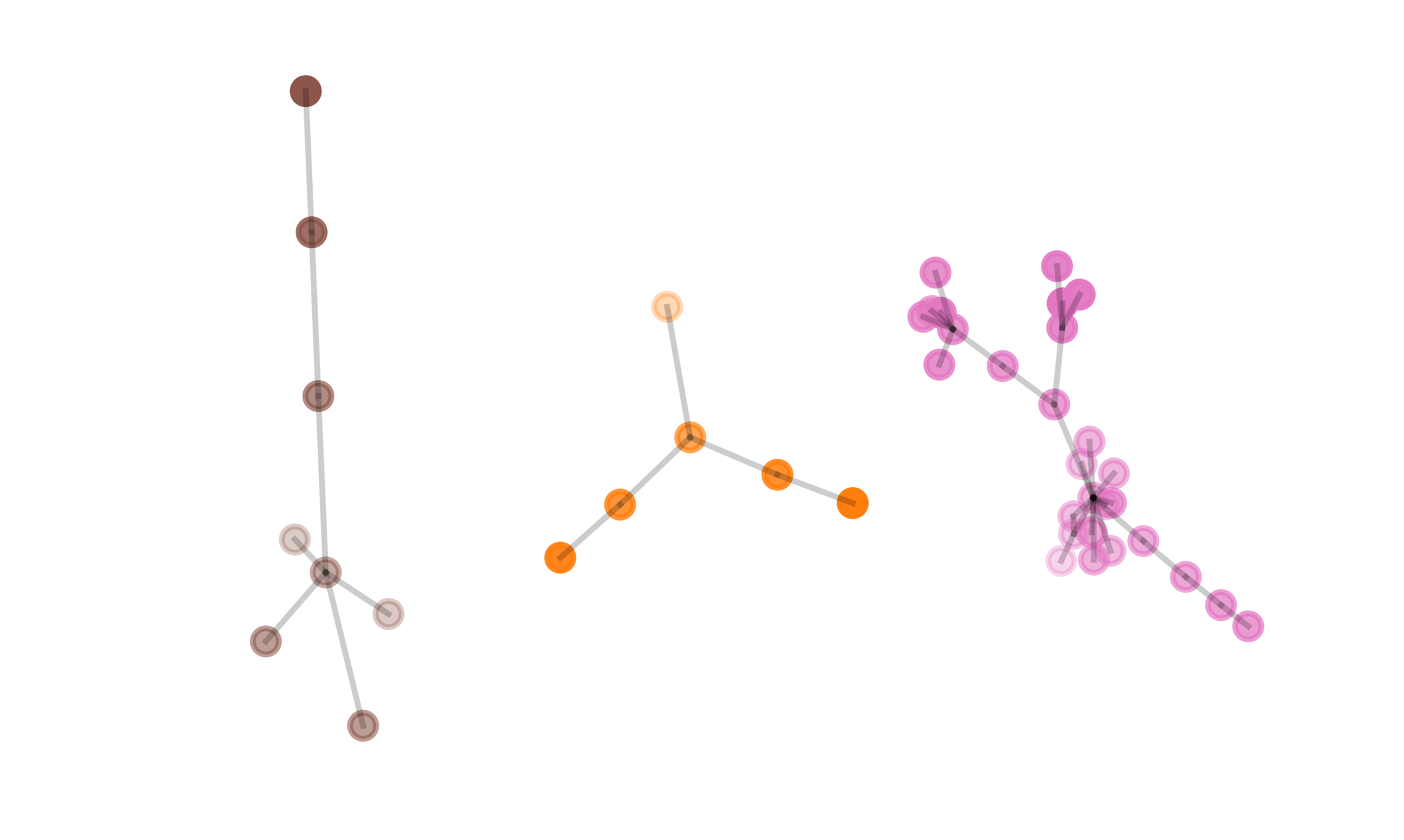}} \\
		\hline
		
	\end{tabular}
	\caption{Routine graphs when patients are removed from the dataset, i.e. PEP method. BRM's
		focus on rare rules highlights the importance of each individual patient's contribution to the graph representation. With PEP analysis, we were able to separate patients into active and sedentary behaviour groups. In contrast, the number of FRM-mined
		routines did not change under PEP analysis and no further insight was derived. }
	\label{fig_LOPO}
\end{figure*}

Fig.~\ref{fig_LOPO} illustrates an example of the changes in routine graphs
when removing patients with active and sedentary behaviour. In the Appendix, we
have included the complete set of graphs for the PEP analysis. For both mining
algorithms, when removing one patient, the routine's activity composition
varied, but the assignment of routine labels by the study observer did not vary.
In the PEP analysis, using the best confidence per conclusion filter, BRM mined
routines that grouped patients into physically active and sedentary groups. We
observed that the removal of patient ID 10 made the routine \emph{using a phone}
disappear. Apparently, the patient frequently used the phone. We found that the
physically active group contained patient IDs 2,4,6,9, and 10. The physically
active group refers to patients following their rehabilitation schedule closely.
No relation to activity intensity or EBI score was found. The active patient
group consisted of one wheelchair rider, patients with different EBI starting
points, and some patients, where the EBI score did not change. When a patient
from the active group was removed, the extracted routine number reduced to four
and \emph{socialising} was always missing. The result appears counter-intuitive
as the sedentary group has been likely involved in socialising, but could be
explained by the chosen 20 minute observation window, which causes BRM to focus
on transitions between activities of at most 20 minute duration. Sedentary
patients would perform individual activities for periods longer than 20 minutes.
Therefore, their socialising activities would not be associated into rules.

For comparison, we performed PEP analysis for FRM-extracted routines using
rules with the best confidence per conclusion. We found that the removal of any
one patient did not affect the extracted routines. Therefore, FRM provided no
further insight, and classifying patient's into physical activity groups was not
possible.

%% file: conlusions.tex
\section{Conclusions} 
BRM's parameters can be adjusted using process knowledge.
Selector $s$ is a proportion of conclusions to consider, which does not dilute
if the dataset size increases, and therefore can be calculated as the ratio of the
expected rare and frequent symbol frequencies. The observation window size is determined using the
transition times between events of interest. The prior $p$ does not play a role
in rule selection. In contrast, most FRM-based methods use parameters that are
difficult to define in terms of the processes or dataset properties.

The prior $p$ determines the rule's final belief value. Therefore, comparisons
of absolute belief values across datasets requires attention to the prior values
used with each dataset. With this consideration,  we do not propose belief as a
stand-alone rule interest metric.

The first limitation of BRM is the sporadic association of frequent premises
with infrequent conclusions into irrelevant rules. By using a filtering
stage, we removed the irrelevant rules. However, we found that the filter of choice
depends on the experiment. For example, in our proof-of-concept experiment, we
used a confidence threshold to separate symbols into generating processes, and
for the rehabilitation routine mining experiment, we used the best rule
confidence per conclusion. We theorize that, in general, if the distributions of
symbols in the dataset is not uniform and the mining task is to separate per underlying process, then the confidence threshold filter may be more useful to remove irrelevant rules, whereas  best confidence per conclusion filter is better suited to extract transitions between symbols.

We believe that BRM rules provide meaningful insight, in particular on rarely,
but consistently occurring relations, which may provide application experts with
new hypotheses to investigate.  For example, as seen in the S$\rightarrow$S
scenario of the database mining experiment, BRM provided additional rules over FRM that hint to a
predominantly white ageing population (Race score = 1 $\rightarrow$ Age group
40-150), and to a correlation between low poverty score and high education
levels (Poverty score = 1 $\rightarrow$ Education score = 7 or 8).

In this paper, we presented an exhaustive search implementation of a BRM algorithm.
Nevertheless, he optimisation strategies used for FRM could be adpated to BRM by replacing the minimum-support with increasing belief and the downward-closure property with 
Prop.~\ref{prop_inc_belief} and Prop.~\ref{prop_premise_conjunction_rules}.

BRM is not a replacement for FRM. The application should drive the choice of
algorithm branch.  If the application task is to extract symbol relationships from a
single process, then, FRM is suitable. If instead, the task is to separate
multiple process in the dataset, then BRM is suitable and may provide more
insight over FRM. We summarise the difference between both branches as follows:
FRM focuses on extracting rules that describe the commonalities between
generating processes. In contrast, BRM looks for rules that describe each
process.

The rehabilitation routine mining experiment illustrated the difference between
both association rule mining branches. Routines mined with FRM did not change during PEP analysis. FRM
mined routines that were common to the entire population. With BRM, the routines
changed during PEP analysis,  grouping patients into active and sedentary
groups. Hence, FRM answers the question: what routines are commonly done by all
patients?, and BRM answers the question: what types of patients are there?

In this paper, we defined increasing belief using the  Bayes theorem
recursively. We introduced the BRM branch to the association rule mining taxonomy, where
rules are extracted using increasing belief, and presented an implementation of an exhaustive search BRM algorithm. We showed that BRM does not suffer from support dilution, and that BRM is capable of extracting
rare rules from a dataset. The proof-of-concept and socioeconomic experiments
illustrated how BRM extracted frequent and rare rules. In the rehabilitation
routine mining experiment, we used BRM to mine rules, create routines, and
group patients into active and sedentary groups. Only BRM rules provided
patient grouping information.

%% file: apendix.tex
\appendix

\begin{figure}[H]
{	\centering
	\setlength{\tabcolsep}{1pt}
	\begin{tabular}{|>{\centering\arraybackslash}m{0.19\textwidth} %
			|>{\vspace {4pt}\centering\arraybackslash}m{0.4\textwidth}%
			|>{\vspace {4pt}\centering\arraybackslash}m{0.4\textwidth}|%
		}
		\multicolumn{3}{c}{\includegraphics[width=\textwidth]{figures/legend_frm.png}}\\
		\hline
		{} & BRM  &  FRM \\
		\hline
		All Patients & 
		\Centerstack{\includegraphics[width=0.4\textwidth]{figures/bre_best-c_graph}} &
		\Centerstack{\includegraphics[width=0.4\textwidth]{figures/frm_best-c_graph}} \\
		\hline
		Without user 01 &
		\Centerstack{\includegraphics[width=0.4\textwidth]{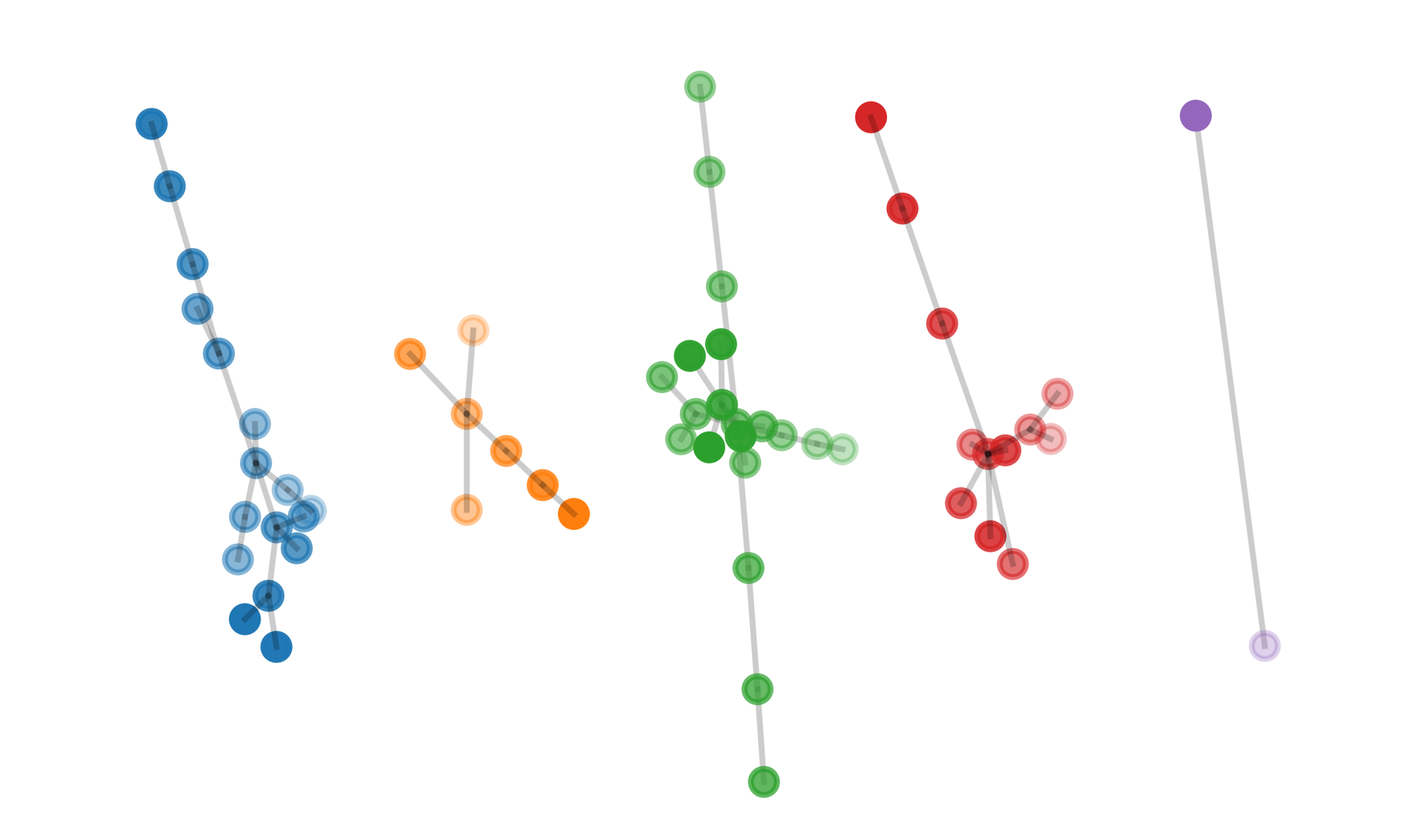}} &
		\Centerstack{\includegraphics[width=0.4\textwidth]{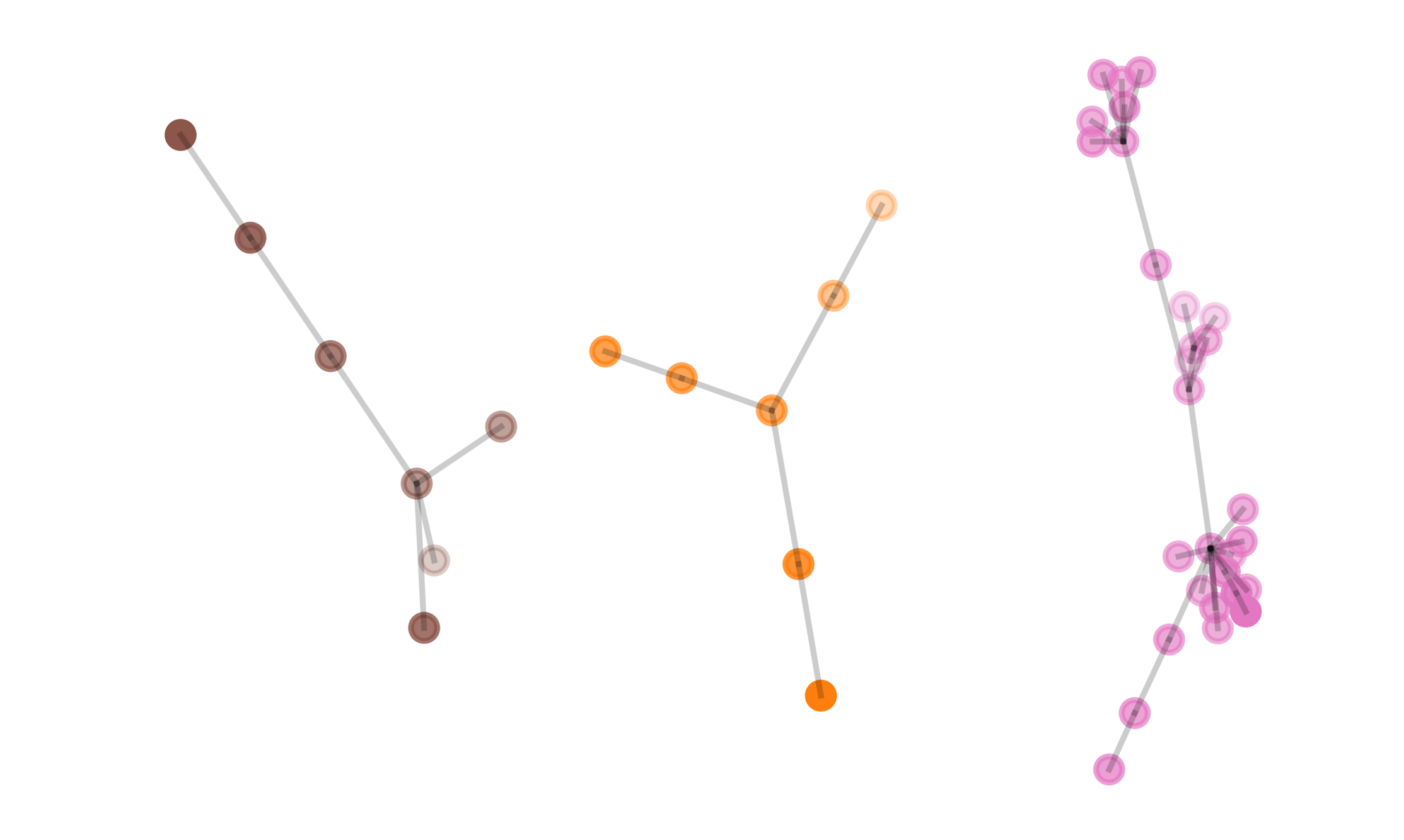}} \\
		\hline
		Without user 02 &
		\Centerstack{\includegraphics[width=0.4\textwidth]{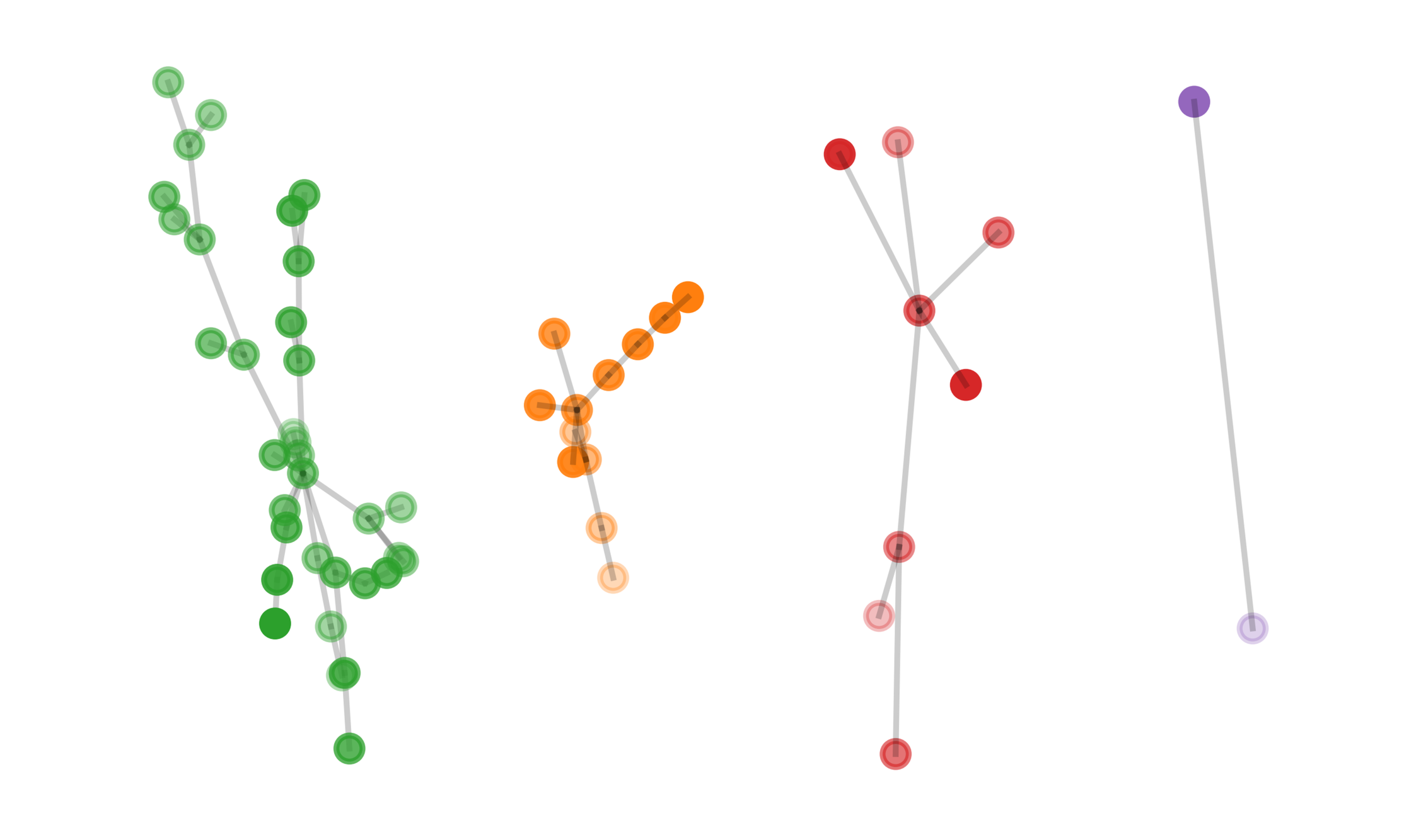}} &
		\Centerstack{\includegraphics[width=0.4\textwidth]{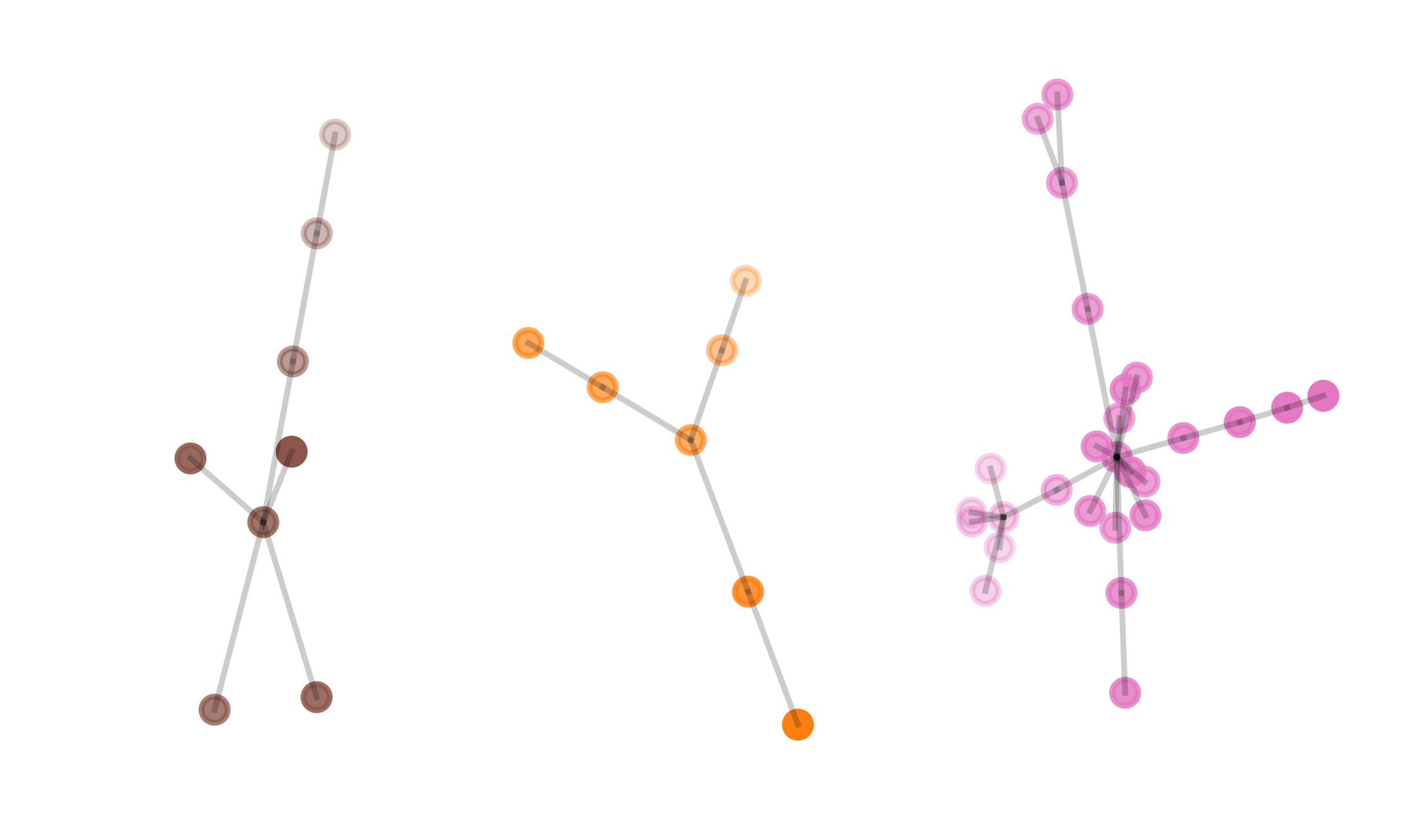}} \\
		\hline
	\end{tabular}}
	\caption{Routine graphs when patients are removed from the dataset, i.e. PEP method. 
		Patients 1 and 2 are grouped as sedentary as there is no change in the number of subgraphs for BRM-mined routines.}
\end{figure}
\clearpage
\newpage
\begin{figure*}
	\centering
    \setlength{\tabcolsep}{1pt}
	\begin{tabular}{|>{\centering\arraybackslash}m{0.19\textwidth} %
					|>{\vspace {4pt}\centering\arraybackslash}m{0.4\textwidth}%
					|>{\vspace {4pt}\centering\arraybackslash}m{0.4\textwidth}|%
				   }
		\multicolumn{3}{c}{\includegraphics[width=\textwidth]{figures/legend_frm.png}}\\
			\hline
			{} & BRM  &  FRM \\
			\hline
			Without user 03 &
			\Centerstack{\includegraphics[width=0.4\textwidth]{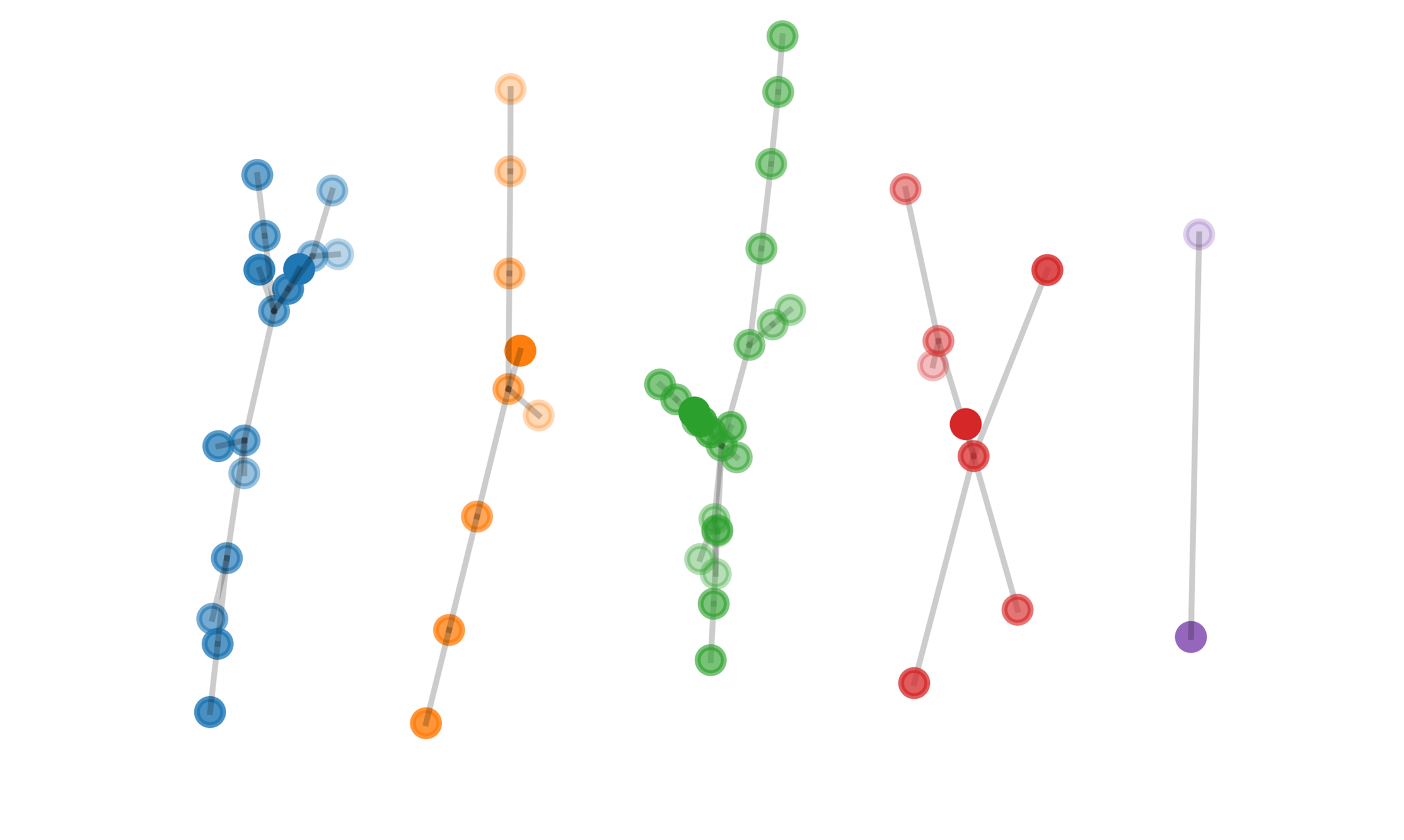}} &
			\Centerstack{\includegraphics[width=0.4\textwidth]{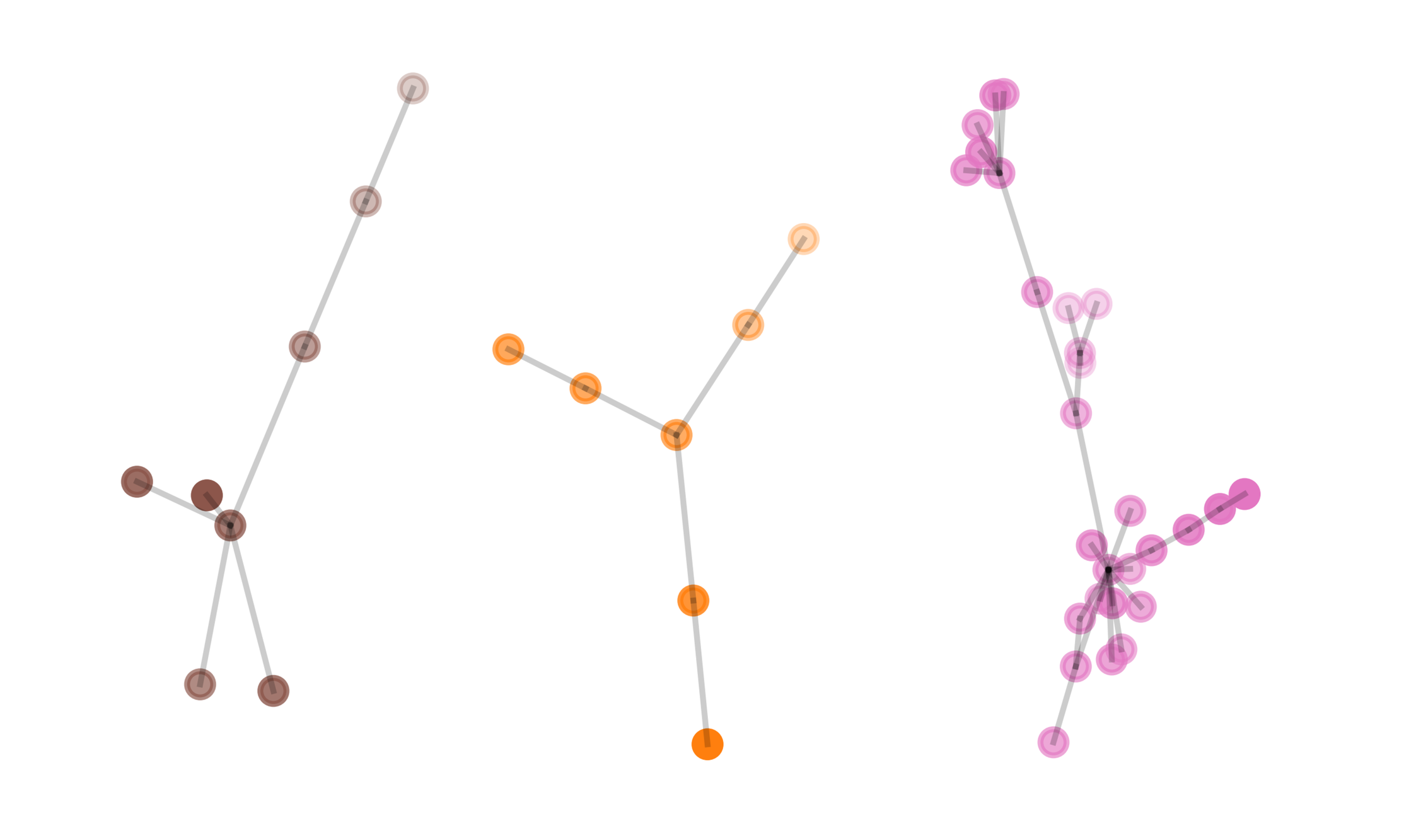}} \\
			\hline
			Without user 04 &
			\Centerstack{\includegraphics[width=0.4\textwidth]{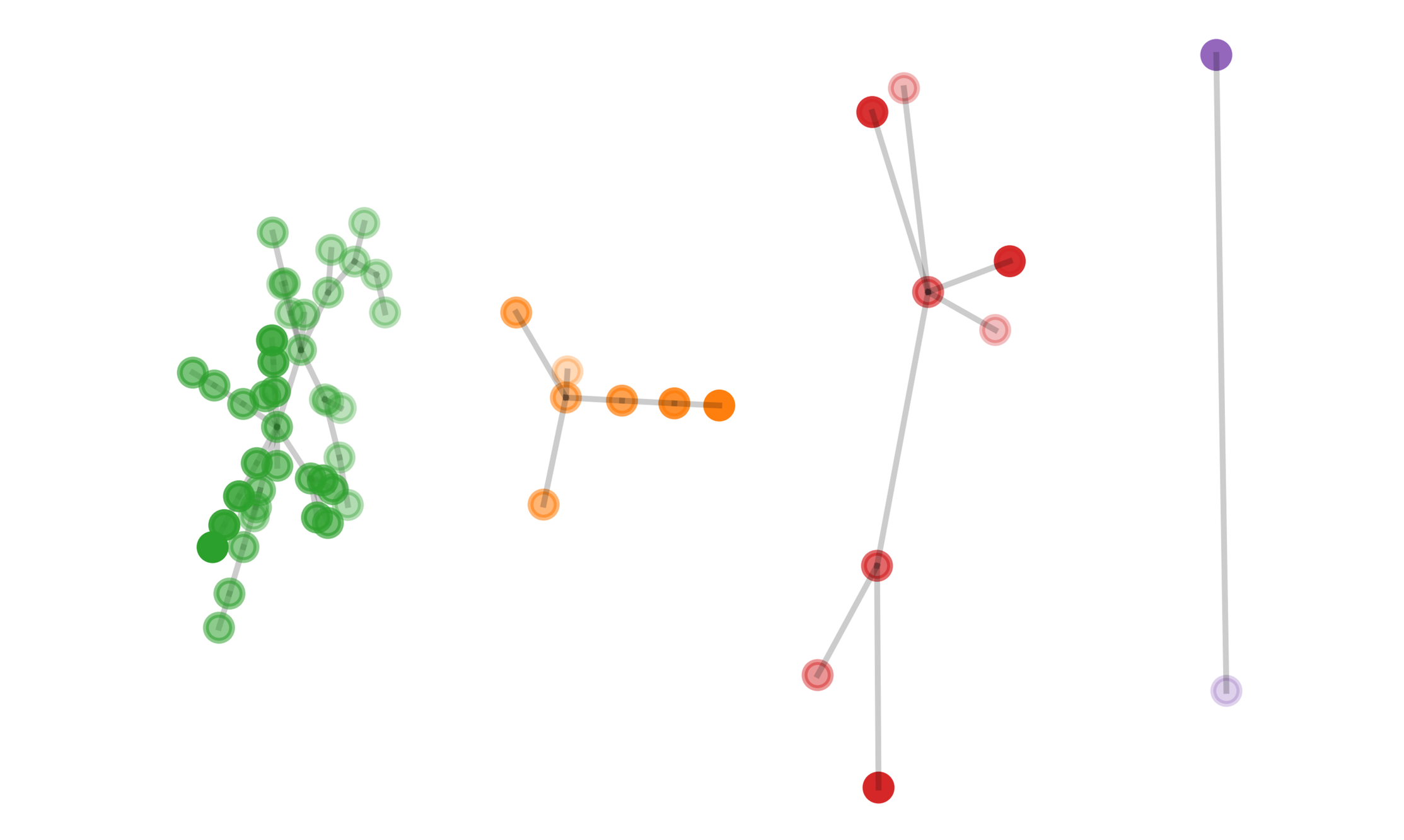}} &
			\Centerstack{\includegraphics[width=0.4\textwidth]{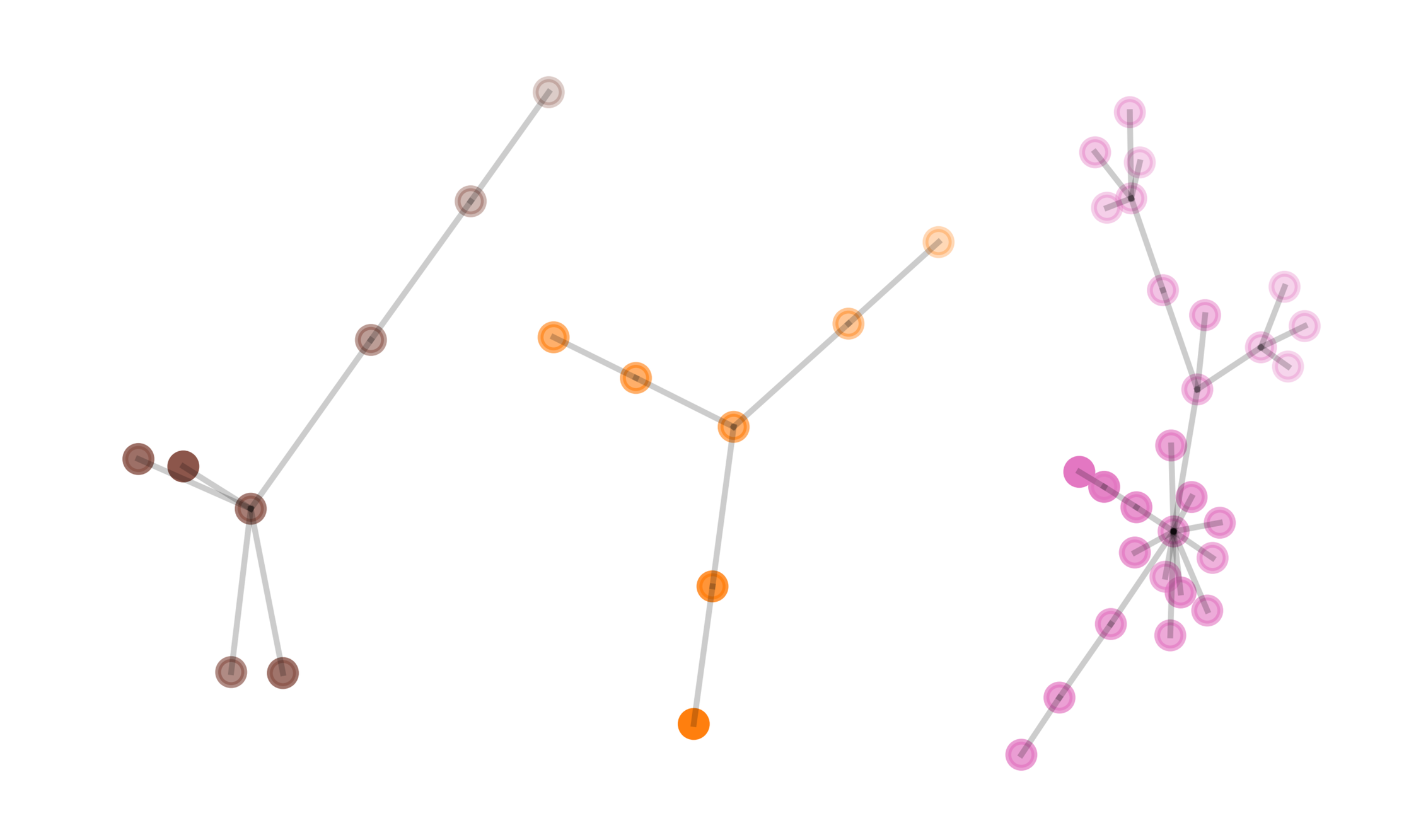}} \\
			\hline
			Without user 05 &
			\Centerstack{\includegraphics[width=0.4\textwidth]{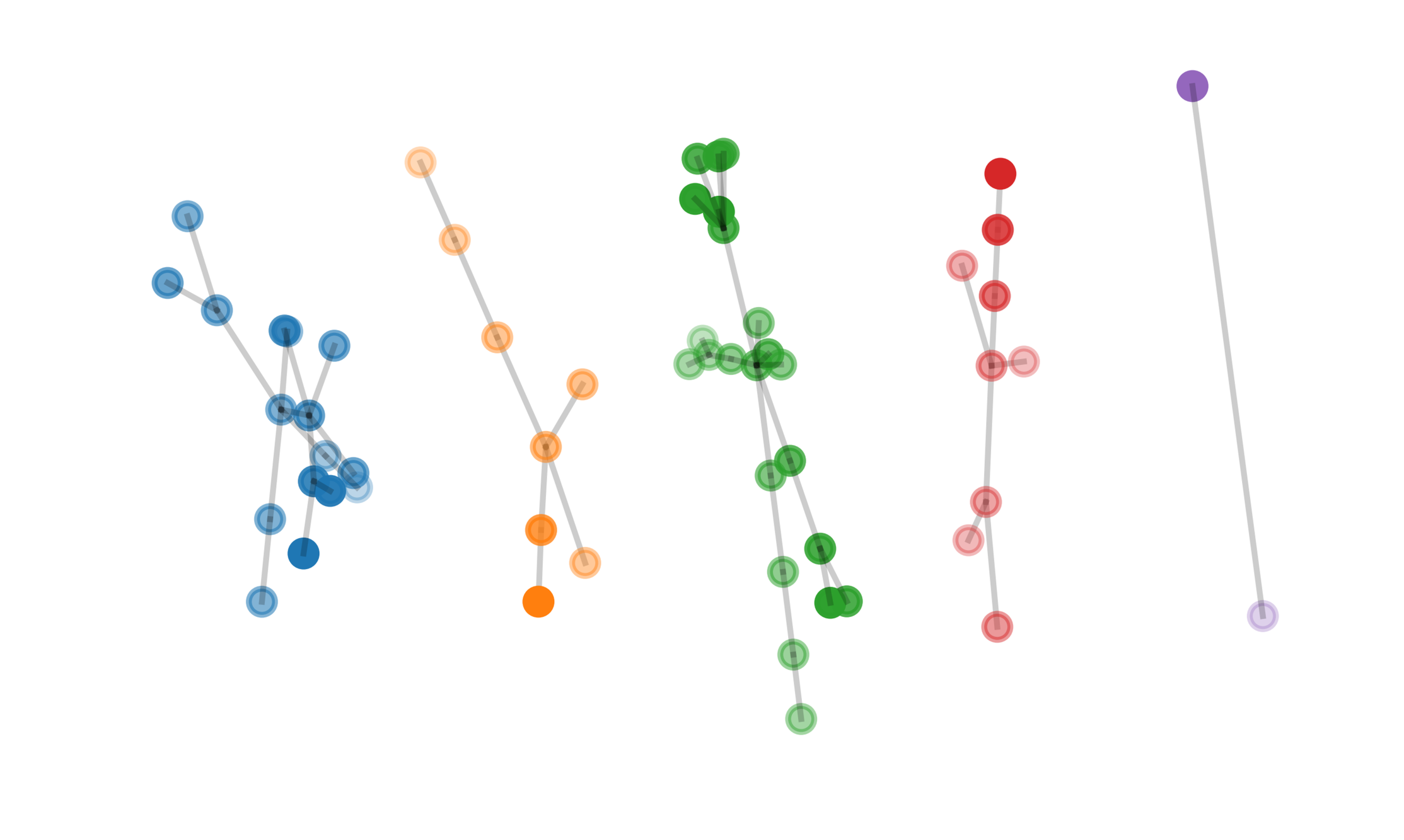}} &
			\Centerstack{\includegraphics[width=0.4\textwidth]{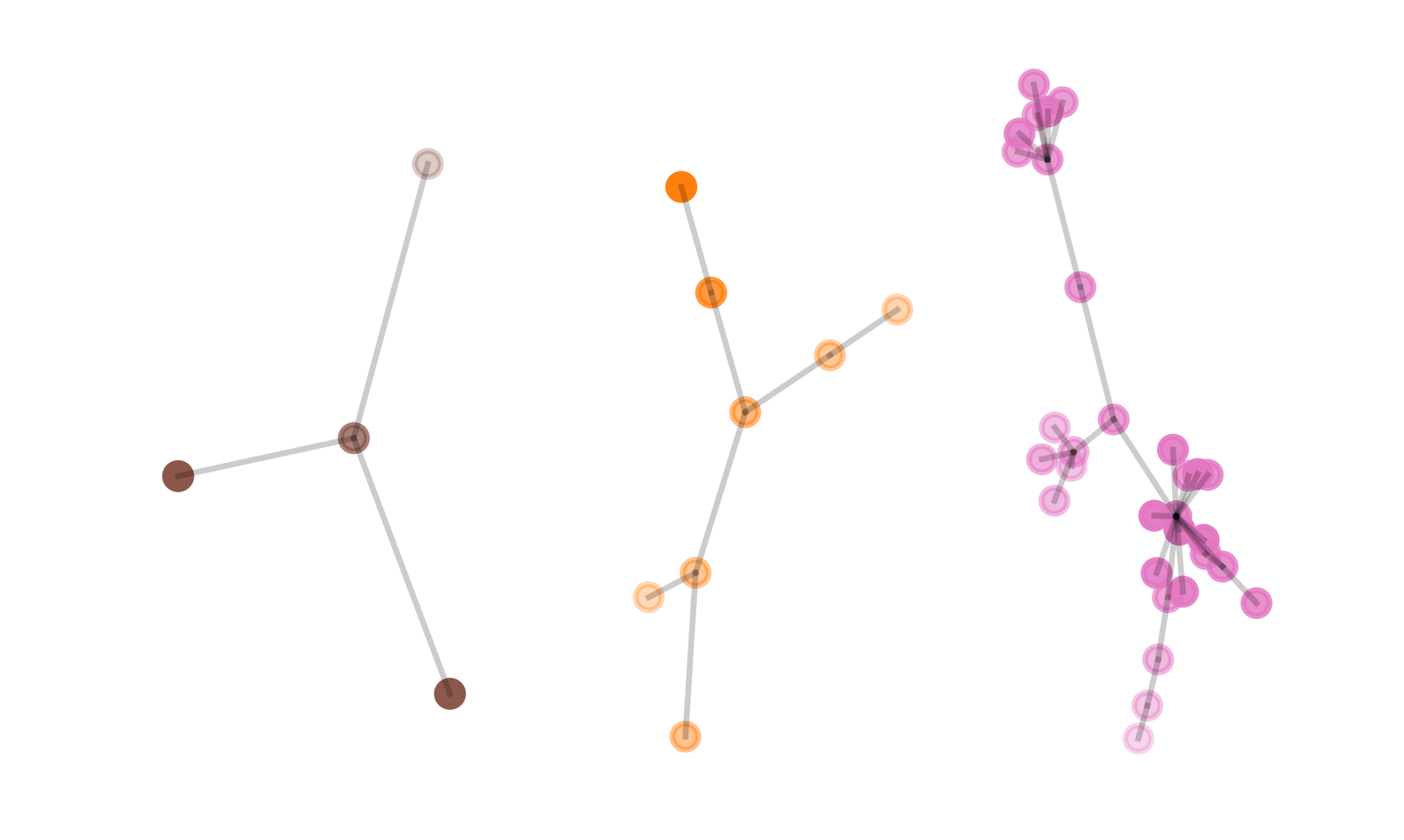}} \\
			\hline
		Without user 06 &
		\Centerstack{\includegraphics[width=0.4\textwidth]{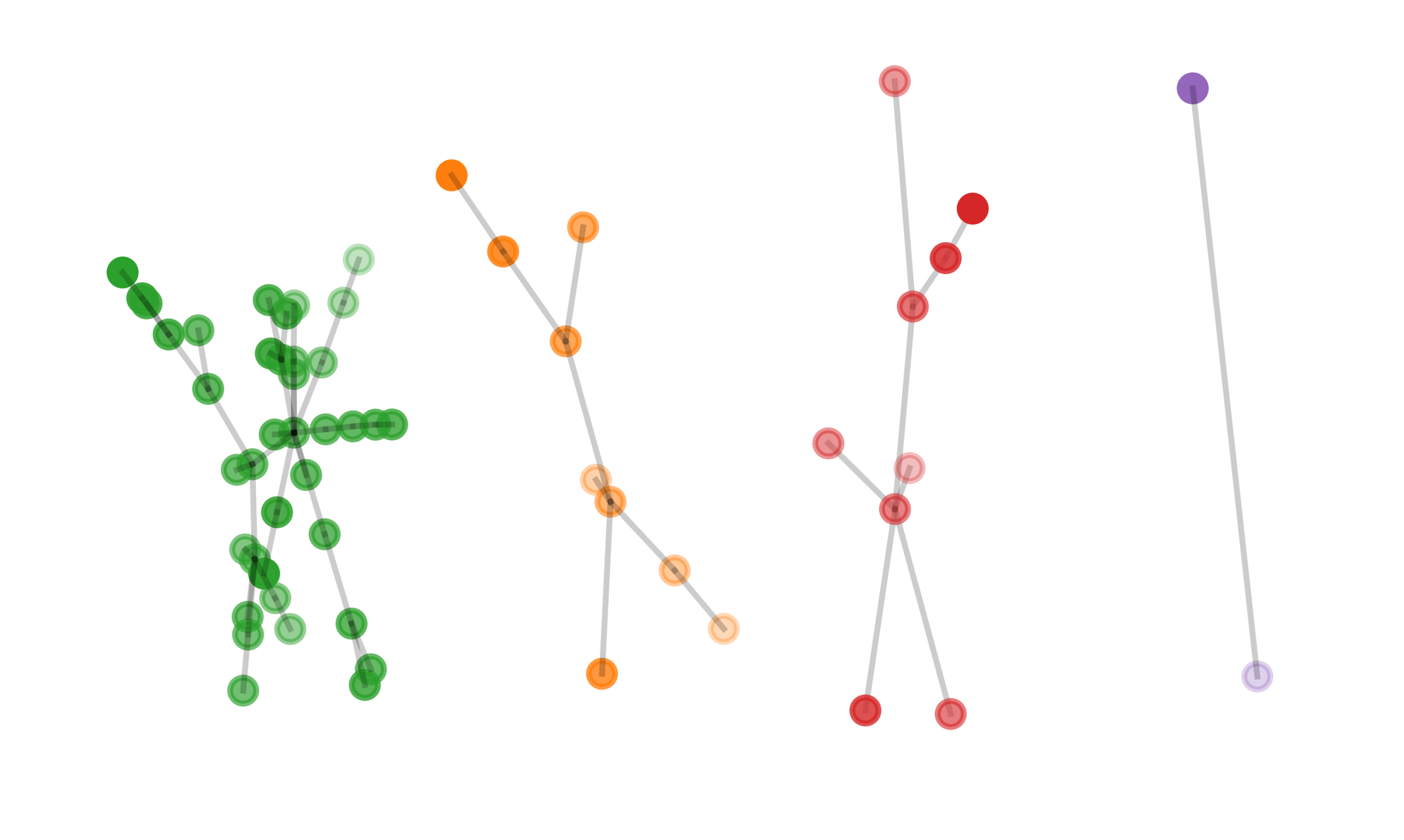}} &
		\Centerstack{\includegraphics[width=0.4\textwidth]{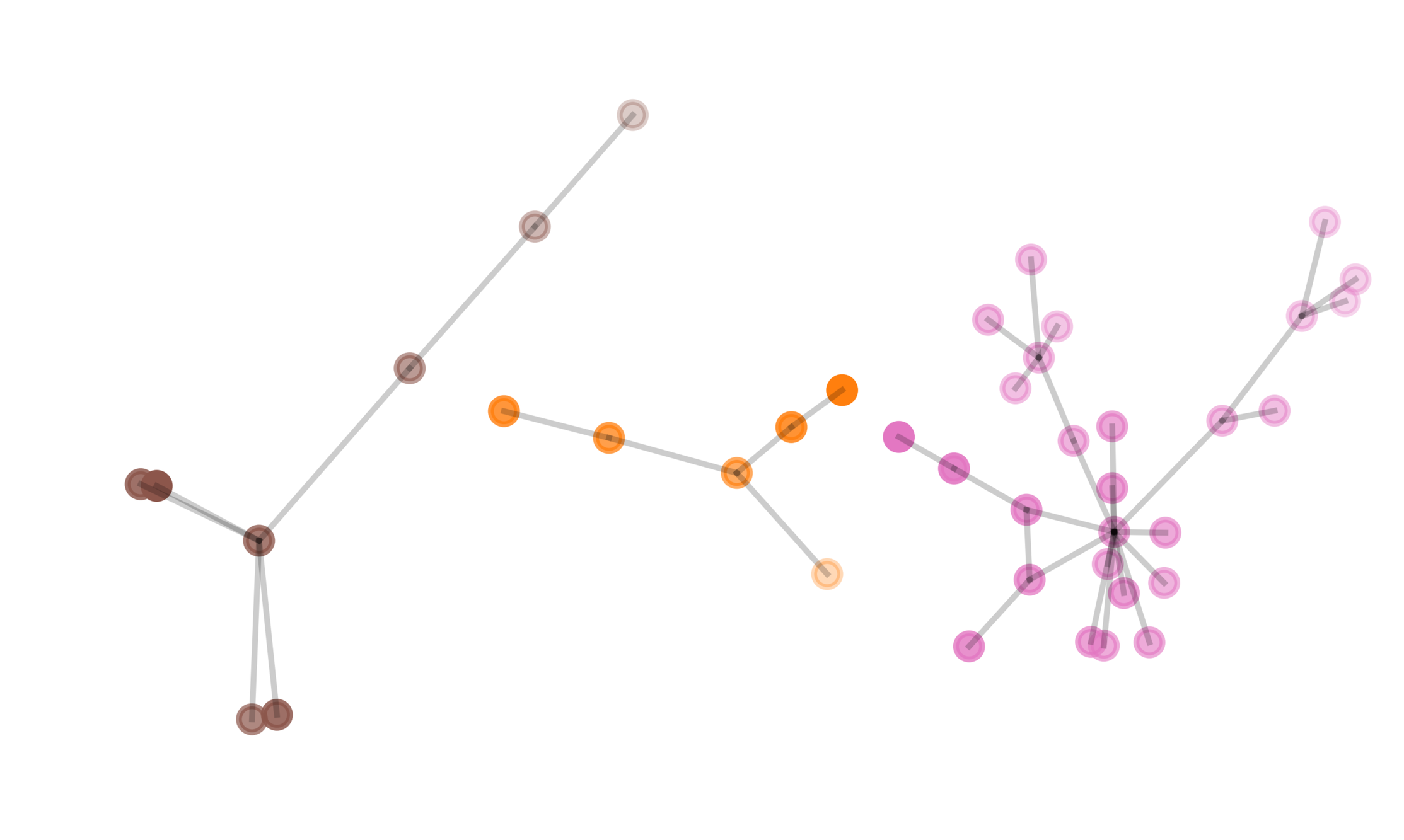}} \\
		\hline
\end{tabular}
\caption{Routine graphs when patients are removed from the dataset, i.e. PEP method. 
		Patient 4 and 6 were grouped with the active patients, as the socializing routine disappeared. Patients 3 and 5 were grouped as sedentary patients, because there is no change in the number of subgraphs for BRM-mined routines.}
\end{figure*}

\clearpage
\begin{figure*}
	\centering
	\setlength{\tabcolsep}{1pt}
	\begin{tabular}{|>{\centering\arraybackslash}m{0.19\textwidth} %
					|>{\vspace {4pt}\centering\arraybackslash}m{0.4\textwidth}%
					|>{\vspace {4pt}\centering\arraybackslash}m{0.4\textwidth}|%
				}
		\multicolumn{3}{c}{\includegraphics[width=\textwidth]{figures/legend_frm.png}}\\
		\hline
		{} & BRM  &  FRM \\
		\hline		
		Without user 07 &
		\Centerstack{\includegraphics[width=0.4\textwidth]{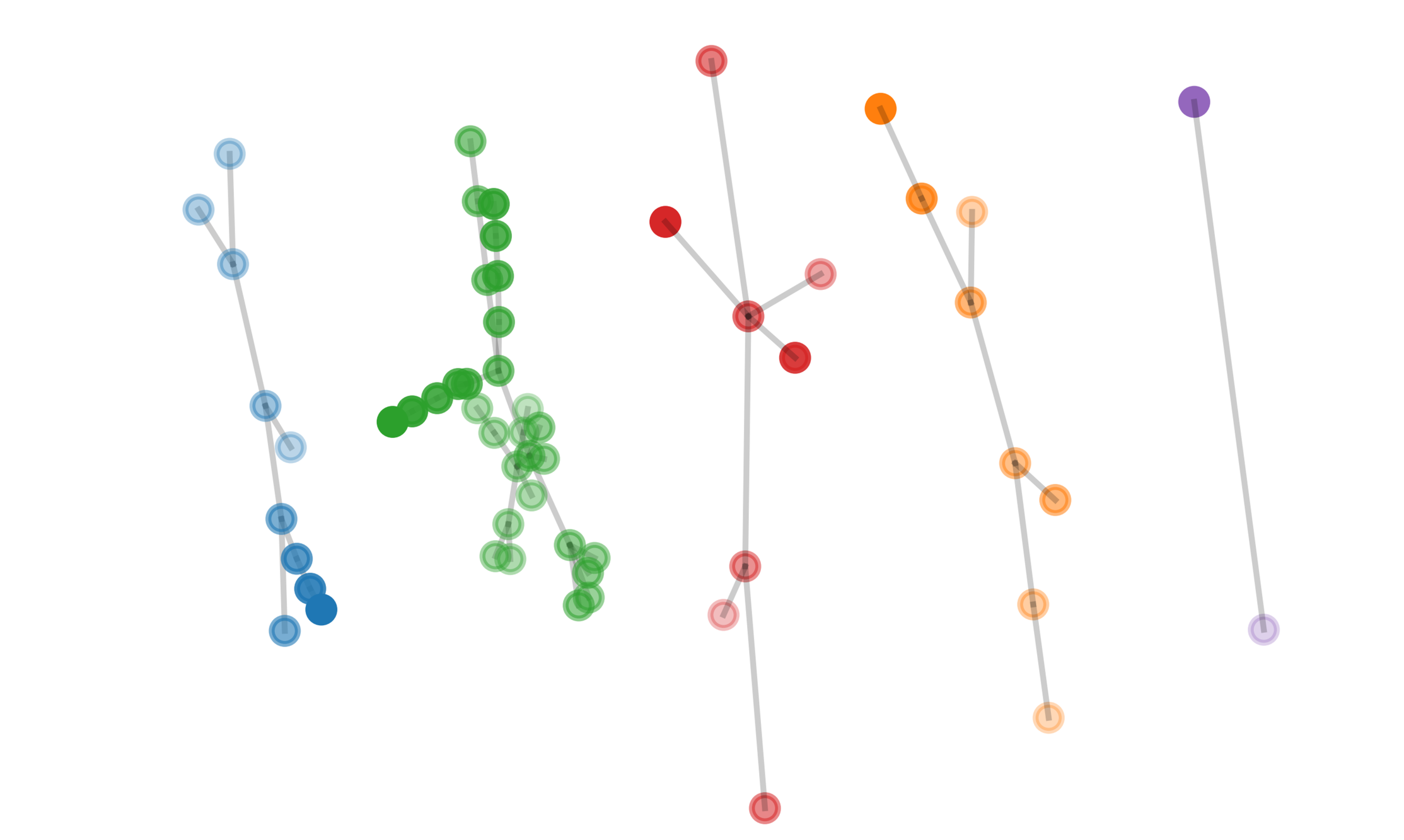}} &
		\Centerstack{\includegraphics[width=0.4\textwidth]{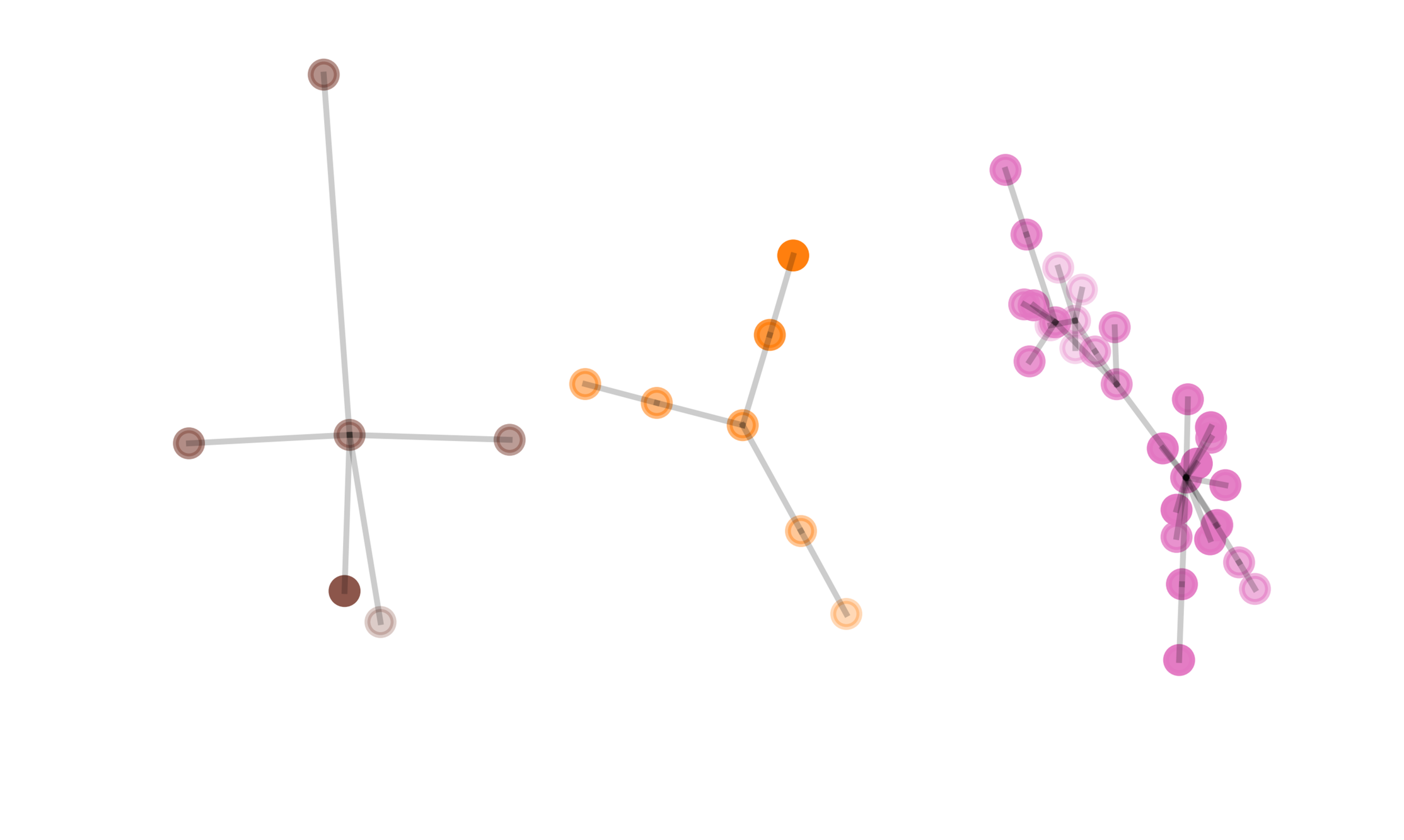}} \\
		\hline
		Without user 08 &
		\Centerstack{\includegraphics[width=0.4\textwidth]{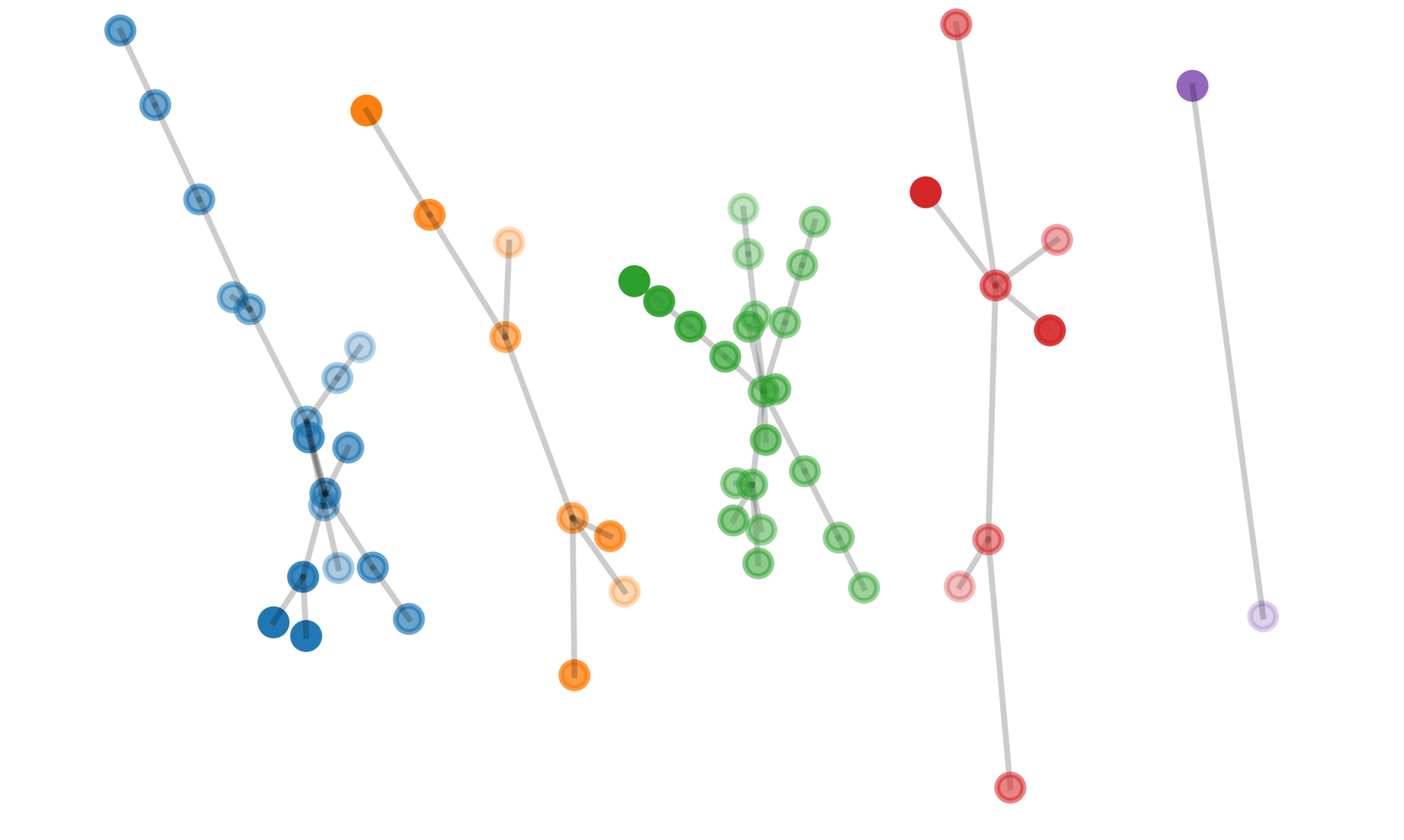}} &
		\Centerstack{\includegraphics[width=0.4\textwidth]{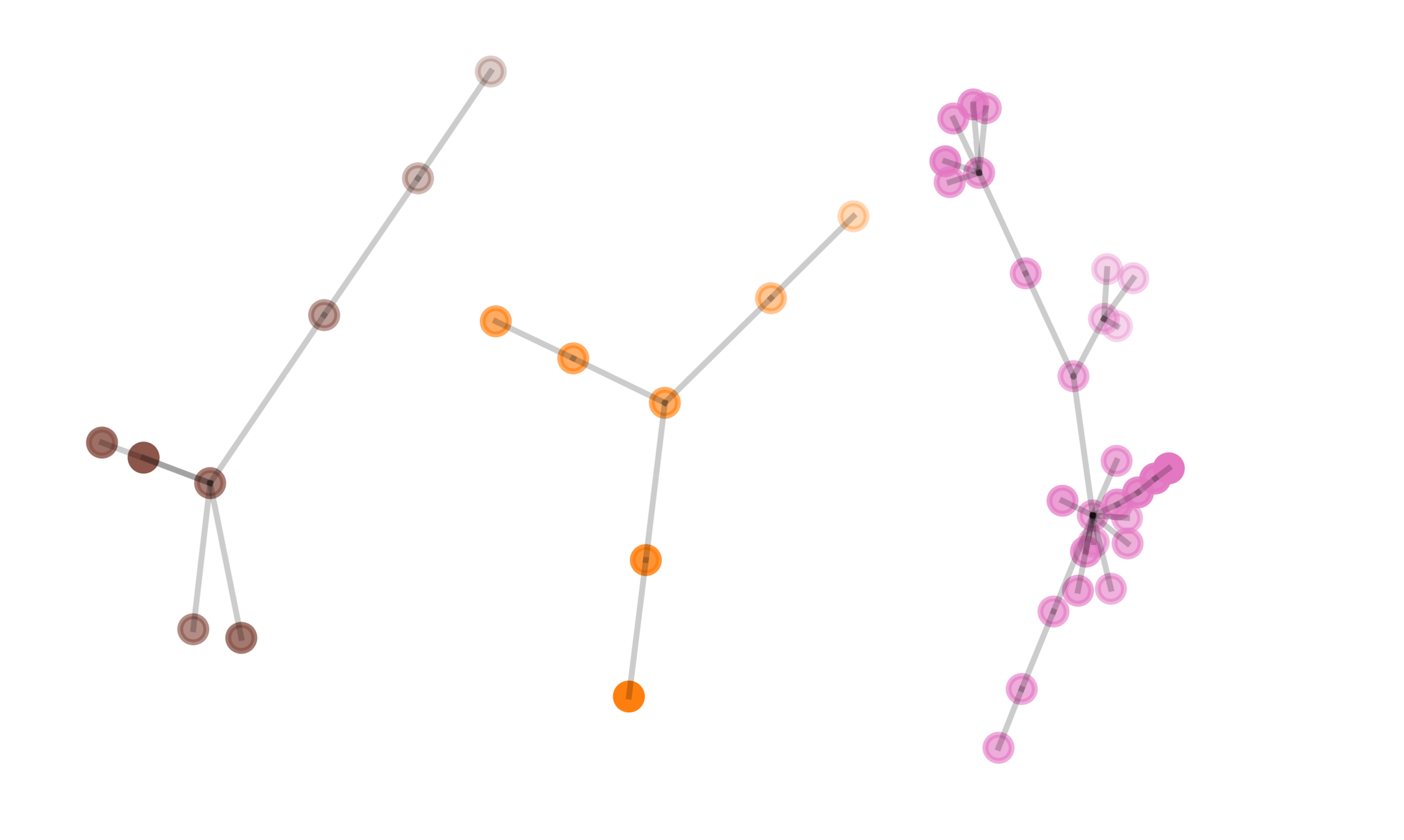}} \\
		\hline
		Without user 09 &
		\Centerstack{\includegraphics[width=0.4\textwidth]{figures/results_lopo_10.png}} &
		\Centerstack{\includegraphics[width=0.4\textwidth]{figures/results_lopo_frm_10.png}} \\
		\hline
	\end{tabular}
	\caption{Routine graphs when patients are removed from the dataset, i.e. PEP method.
		With BRM, when removing patients 9 socialising and intense training merge into one routine graph. Therefore, patient 9 is in the active patient group. Patient 7 and 8 were grouped as sedentary patients, because there is no change in the number of subgraphs for BRM-mined routines.}
\end{figure*}

\clearpage
\begin{figure*}
	\centering
	\setlength{\tabcolsep}{1pt}
	\begin{tabular}{|>{\centering\arraybackslash}m{0.19\textwidth} %
			|>{\vspace {4pt}\centering\arraybackslash}m{0.4\textwidth}%
			|>{\vspace {4pt}\centering\arraybackslash}m{0.4\textwidth}|%
		}
		\multicolumn{3}{c}{\includegraphics[width=\textwidth]{figures/legend_frm.png}}\\
		\hline
		{} & BRM  &  FRM \\
		\hline	
		Without user 10 & 
		\Centerstack{\includegraphics[width=0.4\textwidth]{figures/results_lopo_11.png}} &
		\Centerstack{\includegraphics[width=0.4\textwidth]{figures/results_lopo_frm_11.png}} \\
		\hline	
		Without user 11 &
		\Centerstack{\includegraphics[width=0.4\textwidth]{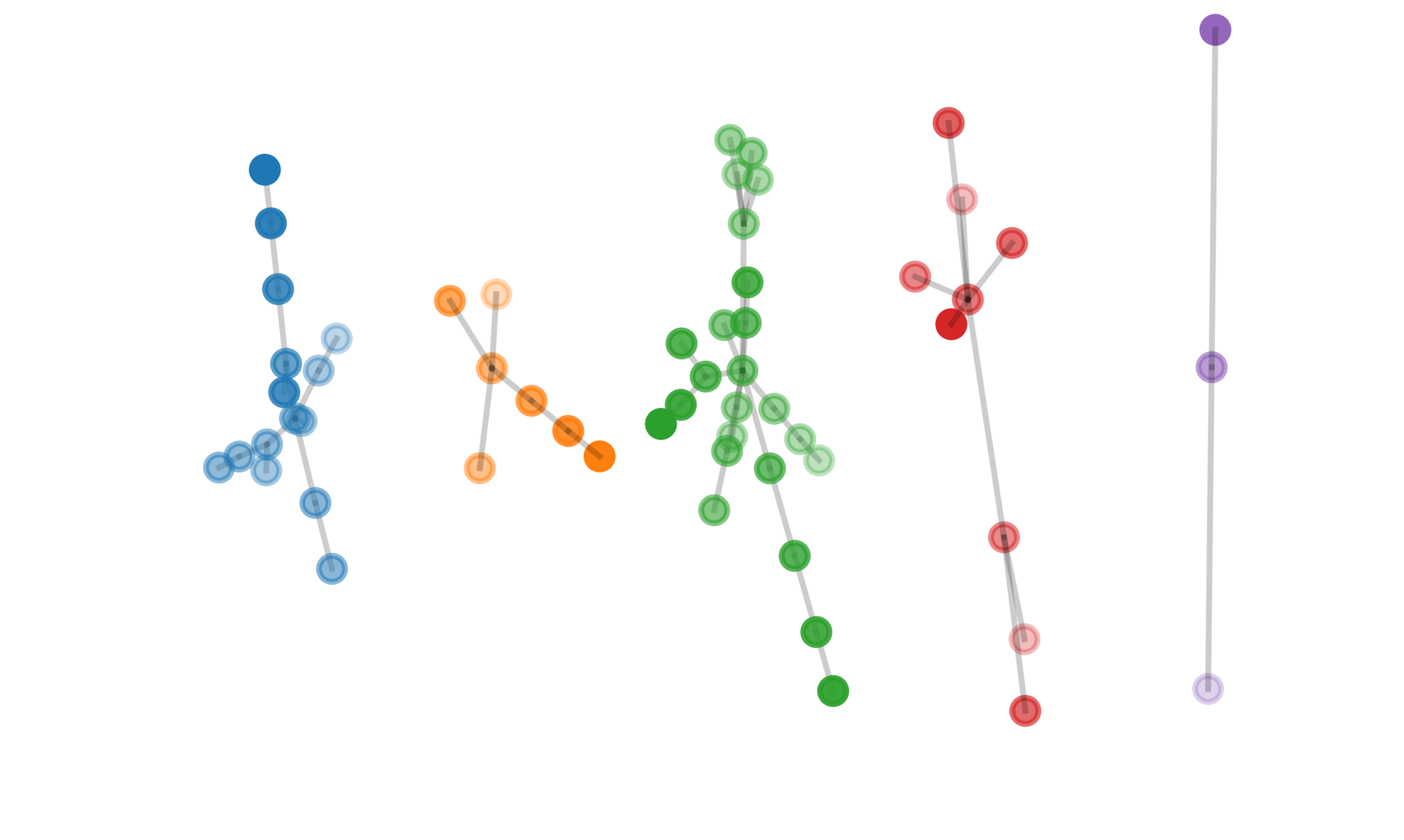}} &
		\Centerstack{\includegraphics[width=0.4\textwidth]{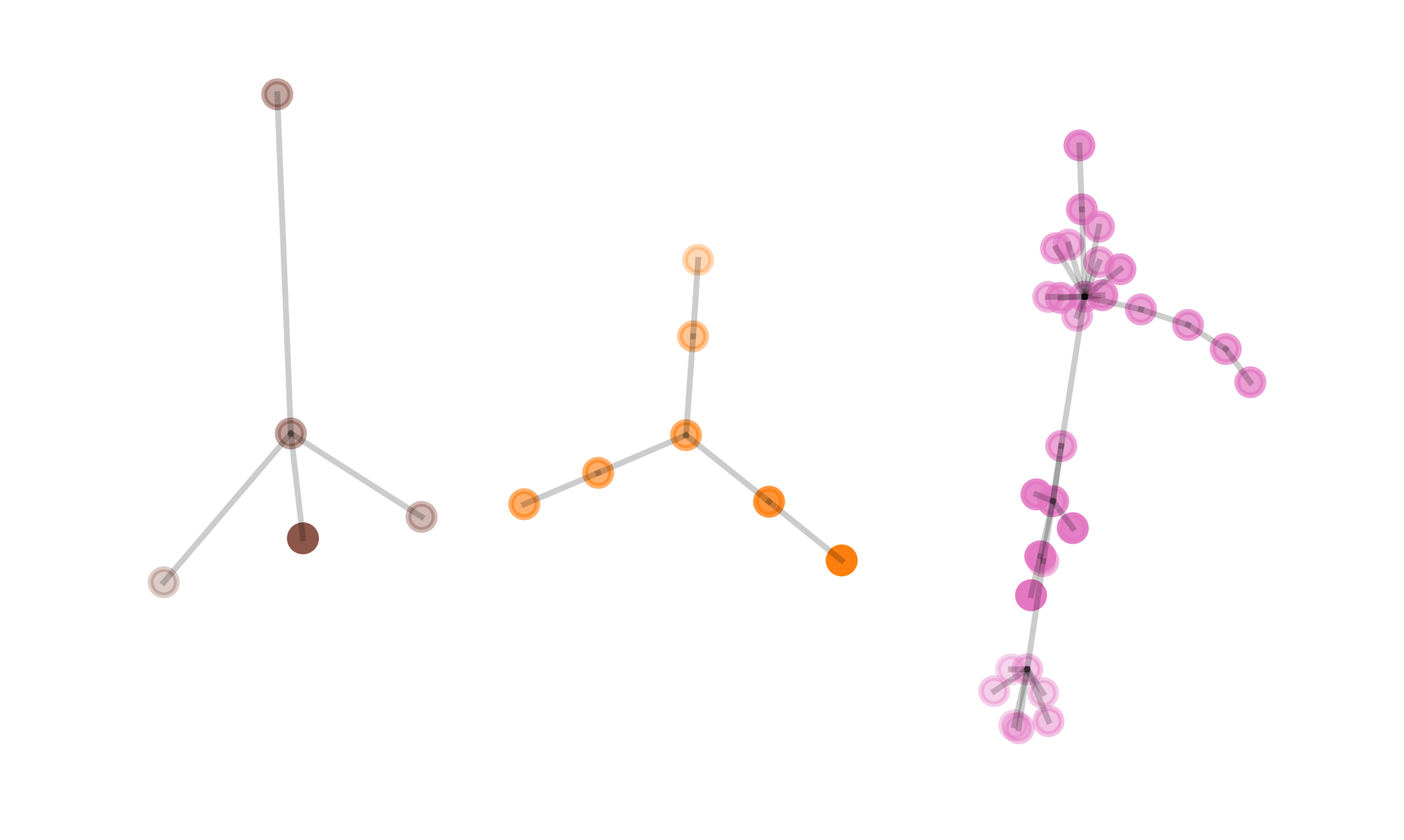}} \\
		\hline
		Without user 12 &
		\Centerstack{\includegraphics[width=0.4\textwidth]{figures/results_lopo_13.png}} &
		\Centerstack{\includegraphics[width=0.4\textwidth]{figures/results_lopo_frm_13.png}}\\
		\hline
	\end{tabular}
	\caption{Routine graphs when patients are removed from the dataset, i.e. PEP method.
		With BRM, when removing patients 10 socialising and intense training merge into one routine, and using phone routine disappears.  Therefore, patient 10 is in the active patient group. Patient 11 and 12 were grouped as sedentary patients, because there is no change in the number of subgraphs for BRM-mined routines.}
\end{figure*}

%% file: ms.bbl
\begin{thebibliography}{10}
\providecommand{\url}[1]{#1}
\csname url@samestyle\endcsname
\providecommand{\newblock}{\relax}
\providecommand{\bibinfo}[2]{#2}
\providecommand{\BIBentrySTDinterwordspacing}{\spaceskip=0pt\relax}
\providecommand{\BIBentryALTinterwordstretchfactor}{4}
\providecommand{\BIBentryALTinterwordspacing}{\spaceskip=\fontdimen2\font plus
\BIBentryALTinterwordstretchfactor\fontdimen3\font minus
  \fontdimen4\font\relax}
\providecommand{\BIBforeignlanguage}[2]{{%
\expandafter\ifx\csname l@#1\endcsname\relax
\typeout{** WARNING: IEEEtran.bst: No hyphenation pattern has been}%
\typeout{** loaded for the language `#1'. Using the pattern for}%
\typeout{** the default language instead.}%
\else
\language=\csname l@#1\endcsname
\fi
#2}}
\providecommand{\BIBdecl}{\relax}
\BIBdecl

\bibitem{Agrawal1993MiningAssociationRulesSets}
R.~Agrawal, T.~Imieli{\'n}ski, and A.~Swami, ``Mining {{Association Rules
  Between Sets}} of {{Items}} in {{Large Databases}},'' in \emph{Proceedings of
  the 1993 {{ACM SIGMOD International Conference}} on {{Management}} of
  {{Data}}}, ser. SIGMOD '93.\hskip 1em plus 0.5em minus 0.4em\relax New York,
  NY, USA: {ACM}, 1993, pp. 207--216.

\bibitem{LoperaGonzalez2016MiningHierarchicalRelationsBuilding}
L.~I. Lopera~Gonzalez and O.~Amft, ``Mining hierarchical relations in building
  management variables,'' \emph{Pervasive Mob. Comput.}, vol.~26, pp. 91--101,
  Feb. 2016.

\bibitem{Liu2018RareItemsetsMiningAlgorithm}
S.~Liu and H.~Pan, ``Rare itemsets mining algorithm based on {{RP}}-{{Tree}}
  and spark framework,'' \emph{AIP Conference Proceedings}, vol. 1967, no.~1,
  p. 040070, May 2018.

\bibitem{Bayes1763EssaySolvingProblemDoctrine}
M.~Bayes and M.~Price, \emph{\BIBforeignlanguage{eng}{An {{Essay}} towards
  {{Solving}} a {{Problem}} in the {{Doctrine}} of {{Chances}}. {{By}} the
  {{Late Rev}}. {{Mr}}. {{Bayes}}, {{F}}. {{R}}. {{S}}. {{Communicated}} by
  {{Mr}}. {{Price}}, in a {{Letter}} to {{John Canton}}, {{A}}. {{M}}. {{F}}.
  {{R}}. {{S}}.}}\hskip 1em plus 0.5em minus 0.4em\relax {Royal Society of
  London}, Jan. 1763.

\bibitem{Beiranvand2014MultiobjectivePSOAlgorithmMining}
V.~Beiranvand, M.~Mobasher-Kashani, and A.~Abu~Bakar,
  ``\BIBforeignlanguage{en}{Multi-objective {{PSO}} algorithm for mining
  numerical association rules without a priori discretization},''
  \emph{\BIBforeignlanguage{en}{Expert Syst. Appl.}}, vol.~41, no.~9, pp.
  4259--4273, 2014.

\bibitem{Padillo2017MiningAssociationRulesBig}
F.~Padillo, J.~Luna, F.~Herrera, and S.~Ventura, ``Mining association rules on
  {{Big Data}} through {{MapReduce}} genetic programming,'' \emph{Integr.
  Comput.-Aided Eng.}, vol.~25, no.~1, pp. 31--48, Dec. 2017.

\bibitem{LoperaGonzalez2018MiningFunctionalStructuralRelationships}
L.~I. Lopera~Gonzalez, ``\BIBforeignlanguage{en}{Mining {{Functional}} and
  {{Structural Relationships}} of {{Context Variables}} in
  {{Smart}}-{{Buildings}}},'' Ph.D. dissertation, Universit{\"a}t Passau,
  Passau, 2018.

\bibitem{Tan2005IntroductionDataMining}
P.-N. Tan, M.~Steinbach, and V.~Kumar, \emph{Introduction to {{Data
  Mining}}}.\hskip 1em plus 0.5em minus 0.4em\relax {Addison-Wesley}, 2005.

\bibitem{Webb1995OPUSEfficientAdmissibleAlgorithm}
G.~I. Webb, ``\BIBforeignlanguage{en-US}{{{OPUS}}: {{An Efficient Admissible
  Algorithm}} for {{Unordered Search}}},''
  \emph{\BIBforeignlanguage{en-US}{1}}, vol.~3, pp. 431--465, Dec. 1995.

\bibitem{Webb2000EfficientSearchAssociationRulesa}
------, ``Efficient {{Search}} for {{Association Rules}},'' in
  \emph{Proceedings of the {{Sixth ACM SIGKDD International Conference}} on
  {{Knowledge Discovery}} and {{Data Mining}}}, ser. KDD '00.\hskip 1em plus
  0.5em minus 0.4em\relax New York, NY, USA: {ACM}, 2000, pp. 99--107.

\bibitem{Agrawal:1994:FAM:645920.672836}
R.~Agrawal and R.~Srikant, ``Fast algorithms for mining association rules in
  large databases,'' in \emph{Proceedings of the 20th International Conference
  on Very Large Data Bases}, ser. VLDB '94.\hskip 1em plus 0.5em minus
  0.4em\relax San Francisco, CA, USA: {Morgan Kaufmann Publishers Inc.}, 1994,
  pp. 487--499.

\bibitem{Zaki2000ScalableAlgorithmsAssociationMininga}
M.~J. Zaki, ``Scalable algorithms for association mining,'' \emph{IEEE Trans.
  Knowl. Data Eng.}, vol.~12, no.~3, pp. 372--390, May 2000.

\bibitem{Bashir2009FastAlgorithmsMiningInteresting}
S.~Bashir, Z.~Jan, and A.~R. Baig, ``Fast {{Algorithms}} for {{Mining
  Interesting Frequent Itemsets}} without {{Minimum Support}},''
  \emph{arXiv:0904.3319 [cs]}, 2009.

\bibitem{Djenouri2018NewFrameworkMetaheuristicbasedFrequent}
Y.~Djenouri, D.~Djenouri, A.~Belhadi, P.~Fournier-Viger, and J.~C.-W. Lin,
  ``\BIBforeignlanguage{en}{A new framework for metaheuristic-based frequent
  itemset mining},'' \emph{\BIBforeignlanguage{en}{Appl Intell}}, vol.~48,
  no.~12, pp. 4775--4791, Dec. 2018.

\bibitem{Tahyudin2017CombinationEvolutionaryAlgorithmMethod}
I.~Tahyudin and H.~Nambo, ``\BIBforeignlanguage{en}{The {{Combination}} of
  {{Evolutionary Algorithm Method}} for {{Numerical Association Rule Mining
  Optimization}}},'' in \emph{\BIBforeignlanguage{en}{Proceedings of the
  {{Tenth International Conference}} on {{Management Science}} and
  {{Engineering Management}}}}, ser. Advances in Intelligent Systems and
  Computing, J.~Xu, A.~Hajiyev, S.~Nickel, and M.~Gen, Eds.\hskip 1em plus
  0.5em minus 0.4em\relax {Springer Singapore}, 2017, pp. 13--23.

\bibitem{GuillameBert2012LearningTemporalAssociationRules}
M.~Guillame-Bert and J.~L. Crowley, ``Learning temporal association rules on
  symbolic time sequences,'' in \emph{Proceedings of the 4th Asian Conference
  on Machine Learning, {ACML} 2012, Singapore, Singapore, November 4-6, 2012},
  2012, pp. 159 -- 174.

\bibitem{Yan2009GeneticAlgorithmbasedStrategyIdentifying}
X.~Yan, C.~Zhang, and S.~Zhang, ``Genetic algorithm-based strategy for
  identifying association rules without specifying actual minimum support,''
  \emph{Expert Syst. Appl.}, vol.~36, no.~2, pp. 3066--3076, 2009.

\bibitem{Brin1997MarketBasketsGeneralizingAssociation}
S.~Brin, R.~Motwani, and C.~Silverstein, ``Beyond {{Market Baskets}}:
  {{Generalizing Association Rules}} to {{Correlations}},'' in
  \emph{Proceedings of the 1997 {{ACM SIGMOD International Conference}} on
  {{Management}} of {{Data}}}, ser. SIGMOD '97.\hskip 1em plus 0.5em minus
  0.4em\relax New York, NY, USA: {ACM}, 1997, pp. 265--276.

\bibitem{Brin1997DynamicItemsetCountingImplication}
S.~Brin, R.~Motwani, J.~D. Ullman, and S.~Tsur, ``Dynamic {{Itemset Counting}}
  and {{Implication Rules}} for {{Market Basket Data}},'' in \emph{Proceedings
  of the 1997 {{ACM SIGMOD International Conference}} on {{Management}} of
  {{Data}}}, ser. SIGMOD '97.\hskip 1em plus 0.5em minus 0.4em\relax New York,
  NY, USA: {ACM}, 1997, pp. 255--264.

\bibitem{Liu2016MiningIntricateTemporalRules}
L.~Liu, S.~Wang, Y.~Peng, Z.~Huang, M.~Liu, and B.~Hu, ``Mining intricate
  temporal rules for recognizing complex activities of daily living under
  uncertainty,'' \emph{Pattern Recognition}, vol.~60, pp. 1015--1028, Dec.
  2016.

\bibitem{Srinivasan2018RuleSelectorSelectingConditionalAction}
V.~Srinivasan, C.~Koehler, and H.~Jin, ``{{RuleSelector}}: {{Selecting
  Conditional Action Rules}} from {{User Behavior Patterns}},'' \emph{Proc ACM
  Interact Mob Wearable Ubiquitous Technol}, vol.~2, no.~1, pp. 35:1--35:34,
  Mar. 2018.

\bibitem{Huang2017AssociationsSociodemographicCharacteristicsChemical}
H.~Huang, R.~Tornero-Velez, and T.~M. Barzyk,
  ``\BIBforeignlanguage{en}{Associations between socio-demographic
  characteristics and chemical concentrations contributing to cumulative
  exposures in the {{United States}}},'' \emph{\BIBforeignlanguage{en}{J. Expo.
  Sci. Environ. Epidemiol.}}, vol.~27, no.~6, pp. 544--550, Dec. 2017.

\bibitem{Derungs2018LongitudinalWalkingAnalysisHemiparetic}
A.~Derungs, C.~Schuster-Amft, and O.~Amft,
  ``\BIBforeignlanguage{English}{Longitudinal {{Walking Analysis}} in
  {{Hemiparetic Patients Using Wearable Motion Sensors}}: {{Is There
  Convergence Between Body Sides}}?}''
  \emph{\BIBforeignlanguage{English}{Front. Bioeng. Biotechnol.}}, vol.~6,
  2018.

\bibitem{Prosiegel1996ErweiterteBarthelIndexEBIeineNeue}
M.~Prosiegel, S.~B{\"o}ttger, T.~Schenk, N.~K{\"o}nig, M.~Marolf, C.~Vaney,
  C.~Garner, and A.~Yassouridis, ``Der erweiterte {{Barthel}}-{{Index}}
  ({{EBI}})\textendash{}eine neue {{Skala}} zur {{Erfassung}} von
  {{F{\"a}higkeitsst{\"o}rungen}} bei neurologischen {{Patienten}},''
  \emph{Neurol Rehabil}, vol.~1, pp. 7--13, 1996.

\end{thebibliography}
